\documentclass[journal, final, 10pt, twocolumn]{IEEEtran}

\IEEEoverridecommandlockouts  

\usepackage{amsmath, amssymb, amsthm}
\usepackage{graphicx, graphics, epsfig, color}
\usepackage{adjustbox}
\usepackage{subfigure}
\usepackage{tikz, pgfplots}
\usetikzlibrary{plotmarks}
\pgfplotsset{compat=1.4}
\usepackage{float}
\usepackage{multirow}
\usepackage{cuted, flushend}
\usepackage{midfloat}
\usepackage{bm}
\usepackage{fancyhdr}
\usepackage{enumerate}
\usepackage{siunitx}
\usepackage[numbers,square,sort&compress]{natbib}
\usepackage{hyperref}
\usepackage{cleveref}
\usepackage{makecell}

\makeatletter
\def\NAT@def@citea{\def\@citea{\NAT@separator}}
\makeatother

\newtheorem{theorem}{Theorem}

\newtheorem{lemma}{Lemma}
\newtheorem{remark}{Remark}
\newtheorem{corollary}{Corollary}

\newtheorem{assumption}{Assumption}

\newcounter{MYtempeqncnt}

\definecolor{mycolor1}{rgb}{0.00000,1.00000,1.00000}%
\definecolor{mycolor2}{rgb}{1.00000,0.00000,1.00000}%

\DeclareMathOperator{\tr}{{\rm tr}}

\newcommand{\asc}{\overset{\rm{a.s.}}{\to}}
\newcommand{\cd}{\overset{d}{\to}}
\newcommand{\ftau}{f(\tau)}
\newcommand{\fftau}{f'(\tau)}
\newcommand{\ffftau}{f''(\tau)}
\newcommand{\iden}{\mathbf{I}_n}
\newcommand{\vnones}{\frac{\mathbf{1}_n\mathbf{1}_n^{\T}}{n}}
\newcommand{\bvp}{\boldsymbol{\varphi}}
\newcommand{\by}{\mathbf{y}}
\newcommand{\bP}{\mathbf{P}}
\newcommand{\bS}{\mathbf{S}}
\newcommand{\bL}{\mathbf{L}}
\newcommand{\bQ}{\mathbf{Q}}
\newcommand{\bw}{\mathbf{w}}
\newcommand{\bx}{\mathbf{x}}
\newcommand{\bk}{\mathbf{k}}
\newcommand{\bK}{\mathbf{K}}
\newcommand{\balpha}{\boldsymbol{\alpha}}
\newcommand{\bpsi}{\boldsymbol{\psi}}
\newcommand{\bmu}{\boldsymbol{\mu}}
\newcommand{\bomega}{\boldsymbol{\omega}}
\newcommand{\bvarpi}{\boldsymbol{\varpi}}
\newcommand{\bOmega}{\boldsymbol{\Omega}}

\newcommand{\bA}{\mathbf{A}}
\newcommand{\bC}{\mathbf{C}}

\newcommand{\db}{\si{\deci\bel}}
\newcommand{\E}{\mathbb{E}}
\newcommand{\bv}{\mathbf{v}}
\newcommand{\bb}{\mathbf{b}}
\newcommand{\bz}{\mathbf{z}}
\newcommand{\bU}{\mathbf{U}}
\newcommand{\bM}{\mathbf{M}}
\newcommand{\bZ}{\mathbf{Z}}
\newcommand{\T}{{\sf T}}
\newcommand{\bLambda}{\boldsymbol{\Lambda}}

\title{A Large Dimensional Analysis of\\ Least Squares Support Vector Machines}

\author{Zhenyu Liao, Romain Couillet
\thanks{This work is supported by the ANR Project RMT4GRAPH (ANR-14-CE28-0006).This paper was presented in part at the 42nd IEEE International Conference on Acoustics, Speech and Signal Processing (ICASSP'17), New Orleans, USA, March 2017.}
\thanks{Z. Liao and R. Couillet are with the Laboratoire de Signaux et Syst\`emes, CNRS-CentraleSup\'elec-Universit\'e Paris-Sud, 3 rue Joliot-Curie, 91192 Gif-sur-Yvette, France (email: zhenyu.liao@l2s.centralesupelec.fr; romain.couillet@centralesupelec.fr).}}

\begin{document}

\date{\today}

\maketitle
\begin{abstract}
	In this article, a large dimensional performance analysis of kernel least squares support vector machines (LS-SVMs) is provided under the assumption of a two-class Gaussian mixture model for the input data. Building upon recent advances in random matrix theory, we show, when the dimension of data $p$ and their number $n$ are both large, that the LS-SVM decision function can be well approximated by a normally distributed random variable, the mean and variance of which depend explicitly on a local behavior of the kernel function. This theoretical result is then applied to the MNIST and Fashion-MNIST datasets which, despite their non-Gaussianity, exhibit a convincingly close behavior. Most importantly, our analysis provides a deeper understanding of the mechanism into play in SVM-type methods and in particular of the impact on the choice of the kernel function as well as some of their theoretical limits in separating high dimensional Gaussian vectors. 
\end{abstract}

\section{Introduction}
\label{sec:intro}

In the past two decades, due to their surprising classification capability and simple implementation, kernel support vector machine (SVM) \cite{cortes1995support} and its variants \cite{suykens1999least,lee2001rsvm,fung2005multicategory} have been used in a wide variety of classification applications, such as face detection \cite{osuna1997training,papageorgiou2000trainable}, handwritten digit recognition \cite{lecun1995learning}, and text categorization \cite{joachims1998text,sculley2007relaxed}. In all aforementioned applications, the dimension of data $p$ and their number $n$ are large: in the hundreds and even thousands. The significance of working in this large $n,p$ regime is even more convincing in the Big Data paradigm today where handling data which are both numerous and large dimensional becomes increasingly common.

Firmly grounded in the framework of statistical learning theory \cite{vapnik2013nature}, support vector machine has two main features: (i) in SVM, the training data $\bx_1,\ldots,\bx_n\in\mathbb{R}^p$ are mapped into some \emph{feature space} through a non-linear function $\bvp$, which, thanks to the so-called ``kernel trick'' \cite{scholkopf2002learning}, needs not be computed explicitly, so that some \emph{kernel function} $f$ is introduced in place of the inner product in the feature space: $f(\bx,\by)=\bvp(\bx)^{\T}\bvp(\by)$, and (ii) a standard (convex) optimization method is used to find the classifier that both minimizes the training error and yields a good generalization performance for unknown data.

As the training of SVMs involves a quadratic programming problem, the computation complexity of SVM training algorithms can be intensive when the number of training examples $n$ becomes large (at least quadratic with respect to $n$). It is thus difficult to deal with large scale problems with traditional SVMs. To cope with this limitation, least squares SVM (LS-SVM, also later referred to as kernel regularized least-squares estimator or kernel ridge regression \cite{murphy2012machine,caponnetto2007optimal,steinwart2009optimal}) was proposed in \cite{suykens1999least}, providing a more computationally efficient implementation of the traditional SVMs, by taking equality optimization constraints instead of inequalities, which results in an explicit solution (from a set of linear equations) rather than an implicit one in SVMs. This article is mostly concerned with this particular type of SVMs.

Trained SVMs are strongly data-dependent: the data with generally unknown statistics are passed through a nonlinear kernel function $f$ and standard optimization methods are used to find the best classifier. All these features make the performance of SVM hardly traceable (at least within the classical finite $n,p$ regime). To understand the mechanism of SVMs, the notion of VC dimension was introduced to provide bounds on the generalization performance of SVM \cite{vapnik2013nature}, while a probabilistic interpretation of LS-SVM was discussed in \cite{gestel2002bayesian} through a Bayesian inference approach. In other related works, connections between LS-SVMs and SVMs were revealed in \cite{ye2007svm}, and more relationships were shown between SVM-type and other learning methods, e.g.,  LS-SVMs and extreme learning machines (ELMs) \cite{huang2012extreme}; SVMs and regularization networks (RNs) \cite{evgeniou2000regularization}, etc. Theoretical analyses on the generalization performance of LS-SVM have been developed, under the conventional asymptotic statistics framework (i.e., assuming $n \to \infty$), to obtain optimal convergence rates in \cite{caponnetto2007optimal,steinwart2009optimal}. Nonetheless, a proper adaptation to the large $n,p$ setting to address LS-SVM performance for large dimensional datasets (of growing interest today) is still missing. 


Similar to classical analysis of asymptotic statistics where $n\to\infty$ while $p$ is fixed, where the diversity of the number of data provides convergence through laws of large numbers, working in the large $n,p$ regime by letting in addition $p\to\infty$ helps exploit the diversity offered by the size of each data vector, providing us with another dimension to guarantee the convergence of some key objects in our analysis, and thus makes the asymptotic analysis of the elusive \emph{kernel matrix} $\bK=\left\{f\left(\bx_i,\bx_j\right)\right\}_{i,j=1}^{n}$ technically more accessible. Recent breakthroughs in random matrix theory have allowed one to overtake the theoretical difficulty posed by the nonlinearity of the aforementioned kernel function $f$ \cite{el2010spectrum,couillet2016kernel} and thus make an in-depth analysis of LS-SVM possible in the large $n,p$ regime. These tools were notably used to assess the performance of the popular Ng-Weiss-Jordan kernel spectral clustering methods for large datasets \cite{couillet2016kernel}, in the analysis of graphed-based semi-supervised learning \cite{mai2017random} or for the development of novel kernel subspace clustering methods \cite{couillet2016random}.

Similar to these works, in this article, we provide a performance analysis of LS-SVM, in the regime of $n,p\to\infty$ and $p/n\to \bar c_0\in(0,\infty)$, under the assumption of a two-class Gaussian mixture model of means $\bmu_1,\bmu_2$ and covariance matrices $\bC_1,\bC_2$ for the input data. The Gaussian assumption may seem artificial to the practitioners, but reveals first insights into how SVM-type methods deal with the information in means and covariances from a more quantitative point of view. Besides, the early investigations \cite{couillet2016kernel,mai2017random} have revealed that the behavior of some machine learning methods under Gaussian or deterministic practical input datasets are a close match, despite the obvious non-Gaussianity of the latter.

Our main finding is that, as in \cite{couillet2016kernel}, in the large $n,p$ regime and under suitable conditions on the input statistics, a non-trivial asymptotic classification error rate (i.e., neither 0 nor 1) can be obtained and the decision function of LS-SVM converges to a Gaussian random variable whose mean and variance depend on the statistics of the two different classes as well as on the behavior of the kernel function $f$ evaluated at $2\tr(n_1\bC_1+n_2\bC_2)/(np)$, with $n_1$ and $n_2$ the number of instances in each class. This brings novel insights into some key issues of SVM-type methods such as kernel function selection and parameter optimization (see for example \cite{gestel2002bayesian,cherkassky2004practical,chapelle2002choosing,ayat2005automatic,weston2001feature,huang2006ga} and the references therein), as far as large dimensional data are concerned. More importantly, we confirm through simulations that our theoretical findings closely match the performance obtained on the MNIST \cite{lecun1998gradient} and the Fashion-MNIST datasets \cite{xiao2017fashion}, which conveys a strong applicative motivation for this work.

\medskip

In the remainder of the article, we provide a rigorous statement of our main results. The problem of LS-SVM is discussed in Section~\ref{sec:problem} and our model and main results presented in Section~\ref{sec:main}, while all proofs are deferred to the appendices in the Supplementary Material. In Section~\ref{sec:special}, attention will be paid on some special cases that are more analytically tractable. Section~\ref{sec:conclusion} concludes the paper by summarizing the main results and outlining future research directions.

\medskip

\emph{Reproducibility}: Python 3 codes to reproduce the results in this article are available at \href{https://github.com/Zhenyu-LIAO/RMT4LSSVM}{https://github.com/Zhenyu-LIAO/RMT4LSSVM}.

\medskip

\emph{Notations}: Boldface lowercase (uppercase) characters stand for vectors (matrices), and scalars non-boldface respectively. $\mathbf{1}_{n}$ is the column vector of ones of size $n$, $\mathbf{0}_n$ the column vector of zeros, and $\mathbf{I}_{n}$ the $n\times n$ identity matrix. The notation $(\cdot)^{\T}$ denotes the transpose
operator. The norm $\|\cdot\|$ is the Euclidean norm for vectors
and the operator norm for matrices. The notation ${\rm P}(\cdot)$ denotes the probability measure of a random variable. The notation $\cd$ denotes convergence in distribution and $\asc$ almost sure convergence, respectively. The operator $\mathcal{D}(\bv)=\mathcal{D}\{v_a\}_{a=1}^k$ is the diagonal matrix having $v_a,\ldots,v_k$ as its ordered diagonal elements. We denote $\{v_a\}_{a=1}^k$ a column vector with $a$-th entry (or block entry) $v_a$ (which may be a vector), while $\{V_{ab}\}_{a,b=1}^k$ denotes a square matrix with entry (or block-entry) $(a,b)$ given by $V_{ab}$ (which may be a matrix). 

\section{Problem statement}
\label{sec:problem}

Least squares support vector machines (LS-SVMs) are a modification of the standard SVM introduced in \cite{suykens1999least} to overcome the drawbacks of SVM related to computational efficiency. The optimization problem has half the number of parameters and benefits from solving a linear system of equations instead of a quadratic programming problem as in standard SVM and is thus more practical for large dimensional learning tasks. In this article, we will focus on a binary classification problem using LS-SVM as described in the following paragraph.

Given a training set $\left\{(\bx_1,y_1),\ldots,(\bx_n,y_n)\right\}$ of size $n$, where data $\bx_i\in\mathbb{R}^p$ and labels $y_i\in\{-1,1\}$, the objective of LS-SVM is to devise a decision function $g(\bx)$ that ideally maps all $\bx_i$ in the training set  to $y_i$ and subsequently all unknown data $\bx$ to their corresponding $y$ value. Here we denote $\bx_i\in\mathcal{C}_1$ if $y_i=-1$ and $\bx_i\in\mathcal{C}_2$ if $y_i=1$ and shall say that $\bx_i$ belongs to class $\mathcal{C}_1$ or class $\mathcal{C}_2$, respectively. Due to the often nonlinear separability of these training data in the input space $\mathbb{R}^p$, in most cases, one associates the training data $\bx_i$ to some feature space $\mathcal{H}$ through a nonlinear mapping $\bvp: \ \bx_i\mapsto\bvp(\bx_i)\in\mathcal{H}$. Constrained optimization methods are then used to define a separating hyperplane in $\mathcal{H}$ with direction vector $\bw$ and correspondingly to find a function $g(\bx)=\bw^{\T}\bvp(\bx)+b$ that minimizes the training errors $e_i=y_i-\left(\bw^{\T}\bvp(\bx_i)+b\right)$, and meanwhile yields good generalization performance by minimizing the norm of $\bw$ \cite{smola2004tutorial}. More specifically, the LS-SVM approach consists in minimizing the squared errors $e_i^2$, thus resulting in\footnote{ We include the bias term $b$ as in the (classical) LS-SVM formulation \cite{suykens1999least}, which may be different from kernel ridge regression in some literature \cite{murphy2012machine,gestel2002bayesian} where no bias term is used.}
\begin{align}\label{eq:LS-SVM-origin}
\underset{\bw,b}{\arg\min} & \quad L(\bw,e)=\|\bw\|^{2}+\frac{\gamma}{n}\sum_{i=1}^{n}e_{i}^{2}\\
\text{such that} & \quad y_{i}=\bw^{\T}\bvp(\bx_{i})+b+e_{i},\ i=1,\ldots,n\notag
\end{align} 
where $\gamma>0$ is a penalty factor that weights the structural risk $\|\bw\|^2$ against the empirical one $\frac{1}{n}\sum_{i=1}^n e_i^2$.

The problem can be solved by introducing Lagrange multipliers $\alpha_i, i=1,\ldots,n$ with solution $\mathbf{{w}}=\sum_{i=1}^{n}\alpha_{i}\bvp(\bx_{i})$, where, letting $\mathbf{{y}}=[y_{1},\ldots,y_{n}]^{\T},\mathbf{\mathbb{\boldsymbol{\alpha}}}=[\alpha_{1},\ldots,\alpha_{n}]^{\T}$, we obtain
\begin{equation}
\begin{cases}
\balpha & ={\bf S}^{-1}\left({\bf I}_{n}-\frac{{\bf 1}_{n}{\bf 1}_{n}^{\T}{\bf S}^{-1}}{{\bf 1}_{n}^{\T}{\bf S}^{-1}{\bf 1}_{n}}\right)\by={\bf S}^{-1}\left({\bf y} - b \mathbf{1}_n \right)\\
b & =\frac{{\bf 1}_{n}^{\T}{\bf S}^{-1}\by}{{\bf 1}_{n}^{\T}{\bf S}^{-1}{\bf 1}_{n}}
\end{cases}\label{eq:LS-SVM-solution-2}
\end{equation}
with $\bS={\bK}+\frac{n}{\gamma}{\bf I}_{n}$ and $\bK\triangleq\left\{\bvp(\bx_i)^{\T}\bvp(\bx_j)\right\}_{i,j=1}^{n}$ referred to as the kernel matrix \cite{suykens1999least}.

Given $\balpha$ and $b$, a new datum $\bx$ is then classified into class $\mathcal{C}_1$ or $\mathcal{C}_2$ depending on the value of the following decision function
\begin{equation}
  g(\bx) = \balpha^{\T}\bk(\bx)+b \label{eq:decision-function}
\end{equation}
where $\bk(\bx)=\left\{\bvp(\bx)^{\T}\bvp(\bx_j) \right\}_{j=1}^{n}\in\mathbb{R}^n$. More precisely, ${\bx}$ is associated to class $\mathcal{C}_{1}$ if $g({\bf x})$ takes a small value (below a certain threshold $\xi$) and to class $\mathcal{C}_{2}$ otherwise.\footnote{Since data from $\mathcal{C}_1$ are labeled $-1$ while data from $\mathcal{C}_2$ are labeled $1$.}

With the ``kernel trick'' \cite{scholkopf2002learning}, as shown in \eqref{eq:LS-SVM-solution-2} and \eqref{eq:decision-function} that, both in the ``training'' and ``testing'' steps, one only needs to evaluate the inner product $\bvp(\bx_i)^{\T}\bvp(\bx_j)$ or $\bvp(\bx)^{\T}\bvp(\bx_j)$, and never needs to know explicitly the mapping $\bvp(\cdot)$. In the rest of this article, we assume that the kernel is \emph{translation invariant} and focus on kernel functions $f:~\mathbb{R}^+\to\mathbb{R}^+$ that satisfy $\bvp(\bx_i)^{\T}\bvp(\bx_j)=f(\|\bx_i-\bx_j\|^2/p)$ and shall redefine $\bK$ and $\bk(\bx)$ for data point $\bx$ as\footnote{As shall be seen later, the division by $p$ here is a convenient normalization in the large $n,p$ regime. For example, we have the (normalized) norm $\| \bx_i\|/\sqrt{p}$ is of order $O(1)$ with high probability for large $n,p$. The motivation of studying ``translation invariant'' kernel is that, being one of the most popular types of kernel used in practice, it offers (additionally) technical tractability as a result of the ``concentration'' phenomenon of large dimensional Gaussian vector, as we shall see later for example in \eqref{eq:concertration}. Similar results can be obtained for ``inner-product'' kernel of the type $f(\bx_i^\T \bx_j/p)$ as presented in \cite{el2010spectrum,ali2018random}.}
\begin{align}
\bK&= \left\{f\left(\|\bx_i-\bx_j\|^2/p\right)\right\}_{i,j=1}^{n}\label{eq:Kernel-Matrix}\\
\bk(\bx)&=\left\{f\left(\|\bx-\bx_j\|^2/p\right)\right\}_{j=1}^{n} \nonumber.
\end{align}

Some commonly used kernel functions are the Gaussian radial basis (RGB) kernel $f(x)=\exp\left(-\frac{x}{2\sigma^2}\right)$ with $\sigma>0$ and the polynomial kernel $f(x)=\sum_{i=0}^{d}{a_i x^i}$ with $d\ge1$. 

In the rest of this article, we will focus on the performance of LS-SVM, in the large $n,p$ regime, by studying the asymptotic behavior of the decision function $g(\bx)$ defined in \eqref{eq:decision-function}, in a binary classification problem with some statistical properties of the data, the model of which will be specified in the next section. 

\section{Main results}
\label{sec:main}
\subsection{Model and assumptions}
\label{sec:model}
Evaluating the performance of LS-SVM is made difficult by the heavily data-driven aspect of the method. In this article, we assume that all $\bx_i$'s are extracted from a Gaussian mixture, thereby allowing for a thorough theoretical analysis.

Let $\bx_{1},\ldots,\bx_{n}\in\mathbb{R}^{p}$ be
independent vectors belonging to two distribution classes $\mathcal{C}_{1},\mathcal{C}_{2}$,
with $\bx_{1},\ldots,\bx_{n_{1}}\in\mathcal{C}_{1}$
and $\bx_{n_{1}+1},\dots,\bx_{n}\in\mathcal{C}_{2}$
(so that class $\mathcal{C}_{1}$ has cardinality $n_{1}$ and class
$\mathcal{C}_{2}$ has cardinality $n-n_1=n_{2}$). We assume that $\bx_{i}\in\mathcal{C}_{a}$
for $a\in\{1,2\}$ if
\[
\bx_{i}=\bmu_{a}+\sqrt{p}\bomega_i
\]
for some $\bmu_{a}\in\mathbb{R}^{p}$ and $\bomega_i\sim\mathcal{N}(0,\bC_a/p)$, with $\bC_{a}\in\mathbb{R}^{p\times p}$ some positive definite matrix.

As the ${\bmu}_a$'s and ${\bC}_a$'s scale with $p$, to avoid asymptotic trivial misclassification rates (i.e., neither $0$ or $1$ in the limit of $n,p \to \infty$), we shall (as in \cite{couillet2018classif,couillet2016kernel}) technically place ourselves under the following controlled growth rate assumption:
\begin{assumption}[Growth Rate] As $n\to\infty$, for $a\in\{1,2\}$, the following conditions hold.\label{as:Growth rate}
  \begin{itemize}
  \item \textbf{Data scaling}: $\frac{p}{n}\triangleq c_0\to\bar{c}_0>0$.
  \item \textbf{Class scaling}: $\frac{n_a}{n}\triangleq c_a\to\bar{c}_a>0$.
  \item \textbf{Mean scaling}: $\|\bmu_2-\bmu_1\|=O(1)$.
  \item \textbf{Covariance scaling}: $\|\bC_a\|=O(1)$ and $\tr(\bC_2-\bC_1)=O(\sqrt{p})$.
  \item for $\bC^\circ\triangleq\frac{n_1}{n}\bC_1+\frac{n_2}{n}\bC_2$, $\frac{2}{p}\tr \bC^\circ\to\tau>0$ as $n,p\to\infty$.
  \end{itemize}
\end{assumption}

From a practical aspect, where $p$ and $n$ are fixed quantities, the dual condition $n\to\infty$ and $\frac{p}{n}\to \bar{c}_0>0$ must be understood as requesting that both $p$ and $n$ be large and such that the ratio $\frac{p}{n}$ is sufficiently distinct from $0$ and $\infty$.\footnote{As a matter of fact, as our results will demonstrate, the case where $\frac{p}{n}\to \bar{c}_0=0$ is also valid as an extension by continuity through $\bar{c}_0\to 0$.}

Aside from the last assumption, stated here mostly for technical convenience, it can be shown that the growth rate demanded in Assumption~\ref{as:Growth rate} is rate-optimal in the sense that an oracle Neyman--Pearson hypothesis testing procedure (with known $\bmu_a$ and $\bC_a$) is (in general) ineffective at any smaller distance rates (so that the misclassification rate will constantly be $1$), as discussed in the following remark.

\begin{remark}[Optimal Growth Rate]
\normalfont Assume that both $\|\bC_a\|$ and $\|\bC_a^{-1}\|$ are of order $O(1)$ and let $\bx$ be a vector belonging to class $\mathcal{C}_1$, i.e., $\bx \sim \mathcal{N}(\bmu_1,\bC_1)$. Then, for perfectly known means $\bmu_1,\bmu_2$ and covariances $\bC_1,\bC_2$, the Neyman--Pearson test for $\bx$ to belong to $\mathcal{C}_1$ consists in the following comparison,
\[
	\left( \bx - \bmu_2 \right)^{\T} \bC_2^{-1} \left( \bx - \bmu_2 \right) - \left( \bx - \bmu_1 \right)^{\T} \bC_1^{-1} \left( \bx - \bmu_1 \right) \lessgtr \log \frac{\det \bC_1}{\det \bC_2}
\]
which is further equivalent to 
\begin{align*}
	t(\bx) &\triangleq \bomega^{\T} \left( \bC_2^{-1} - \bC_1^{-1} \right) \bomega + \frac2{\sqrt{p}} \Delta\bmu^{\T} \bC_2^{-1} \bomega + \frac1p \Delta\bmu^{\T} \bC_2^{-1} \Delta\bmu \\
	&- \frac1p \log \frac{\det \bC_1}{\det \bC_2} \lessgtr 0
\end{align*}
where we denote $\Delta\bmu \triangleq \bmu_1 - \bmu_2$, $\bomega \triangleq \frac1{\sqrt{p}} (\bx - \bmu_1) $ and thus $\bomega \sim \mathcal{N}(0, \bC_1/p)$. To explore the difference in means $\Delta\bmu$ we take $\bC_1 = \bC_2 = \bC$ and by Lyapunov's CLT \cite[Theorem~27.3]{billingsley2008probability} we have, as $p \to \infty$,
\[
	t(\bx) - \hat t \cd 0.
\]
where $\hat t \sim \mathcal{N} \left( \frac1p \Delta\bmu^{\T} \bC^{-1} \Delta\bmu, \frac2p \Delta\bmu^{\T} \bC^{-1} \Delta\bmu \right)$.

For a non-trivial classification rate, the mean of $\hat t$ must scale with $p$ at least at the same rate as its standard deviation and thus, since $\|\bC_a^{-1}\|=O(1)$, this implies that $\|\boldsymbol{\Delta\mu}\|$ be at least of order $O(1)$. Similar analysis can be performed to obtain the rate $\| \bC_1 - \bC_2 \| = O(1/\sqrt{p})$ and consequently $\tr(\bC_2-\bC_1)=O(\sqrt{p})$. We refer the readers to \cite{couillet2018classif} for more discussions in this respect.
\label{rem:optimal-growth-rate}
\end{remark}

\medskip

A key observation, also made in \cite{couillet2016kernel}, is that, as a consequence of Assumption \ref{as:Growth rate}, for all pairs $i\neq j$, 
\begin{equation}
\|\bx_i-\bx_j\|^2/p \asc \tau\label{eq:concertration}
\end{equation}
and the convergence is even uniform across all $i\neq j$. This remark is the crux of all subsequent results (note that, surprisingly at first, it states that all data are essentially at the same distance from one another, irrespective of classes, and that the matrix $\bK$ defined in \eqref{eq:Kernel-Matrix} has all its entries essentially equal ``in the limit'' due to the the high dimensional nature of the data; this can be seen as a manifestation of the ``curse of dimensionality'' with respect to the Euclidean distance in high-dimensional space).

The function $f$ defining the kernel matrix $\bK$ in \eqref{eq:Kernel-Matrix} shall be requested to satisfy the following assumption:
\begin{assumption}[Kernel Function]
\label{as:Kernel function}
The function $f$ is a three-times differentiable function in a neighborhood of $\tau$.
\end{assumption}

The objective of this article is to assess the performance of LS-SVM, under the setting of Assumptions~\ref{as:Growth rate} and~\ref{as:Kernel function}, by studying the asymptotic behavior of the decision function $g(\bx)$ defined in \eqref{eq:decision-function}. Following the work of \cite{el2010spectrum} and \cite{couillet2016kernel}, under our basic settings, the convergence in \eqref{eq:concertration} makes it possible to linearize the kernel matrix $\bK$ around the matrix $\ftau \mathbf{1}_n \mathbf{1}_n^{\T}$, and thus the intractable nonlinear kernel matrix $\bK$ can be asymptotically linearized in the large $n,p$ regime. As such, since the decision function $g(\bx)$ is explicitly defined as a function of $\bK$ (through $\balpha$ and $b$ as defined in \eqref{eq:LS-SVM-solution-2}), one can work out an asymptotic linearization of $g(\bx)$ as a function of the kernel function $f$ and the statistics of the data. This analysis, presented in detail in Appendix~A of the Supplementary Material, allows one to reveal the relationship between the performance of LS-SVM and the kernel function $f$ as well as the given learning task, for Gaussian input data as $n,p\to\infty$, as presented in the following subsection.

\subsection{Asymptotic behavior of the decision function \texorpdfstring{$g(\bx)$}{g(x)}}
\label{asymptotic-g(x)}
Before going into our main results, a few notations need to be introduced. In the remainder of the article, we shall use the following deterministic and random elements notations:
\begin{align*}
	\bP&\triangleq\iden- \mathbf{1}_n \mathbf{1}_n^\T/n \in\mathbb{R}^{n\times n}, \quad \bOmega\triangleq\left[\bomega_1,\ldots,\bomega_n\right]\in\mathbb{R}^{p\times n}\\
	\bpsi&\triangleq\left\{\|\bomega_i\|^2-\mathbb{E}\left[\|\bomega_i\|^2\right]\right\}_{i=1}^{n}\in\mathbb{R}^n.
\end{align*}

Under Assumptions~\ref{as:Growth rate} and \ref{as:Kernel function}, following up \cite{couillet2016kernel}, one can approximate the kernel matrix $\bK$ by $\hat{\bK}$ in such a way that
\[
	\|\bK-\hat{\bK}\|\asc 0 
\]
with $\hat{\bK}=-2\fftau(\mathbf{M}+\mathbf{V}\mathbf{V}^{\T})+\left(f(0)-\ftau+\tau\fftau\right)\iden$ for some matrices $\mathbf{M}$ and $\mathbf{V}$, where $\mathbf{M}$ is a standard random matrix model (of operator norm $O(1)$) and $\mathbf{V}\mathbf{V}^{\T}$ a small rank matrix (of operator norm $O(n)$), which depends both on $\bP,\bOmega,\bpsi$ and on the class statistics $\bmu_1,\bmu_2$ and $\bC_1,\bC_2$. The same analysis is applied to the vector $\bk(\bx)$ by similarly defining the following random variables for a new datum $\bx\in\mathcal{C}_a$, $a\in\{1,2\}$:
\[
	\bomega_\bx \triangleq (\bx-\bmu_a)/\sqrt{p} \in\mathbb{R}^p,\quad \psi_\bx \triangleq \|\bomega_\bx\|^2-\mathbb{E}\left[\|\bomega_\bx\|^2\right]\in\mathbb{R}.
\]
Based on the (operator norm) approximation $\bK\approx\hat{\bK}$, a Taylor expansion is then performed on $\bS^{-1}=\left({\bK}+n\iden/\gamma\right)^{-1}$ to obtain an (asymptotic) approximation of $\bS^{-1}$, and subsequently on $\balpha$ and $b$ which depend explicitly on $\bS^{-1}$. At last, plugging these results into \eqref{eq:decision-function}, one finds the main technical result of this article as follows.

\begin{theorem}[Asymptotic Approximation] \label{thm:Random Equivalent}
Let Assumptions~\ref{as:Growth rate} and \ref{as:Kernel function} hold,
and $g(\bx)$ be defined by \eqref{eq:decision-function}. Then, as $n,p\to\infty$, $n(g({\bf x})-\hat{g}(\bx))\asc 0$,
where
\begin{equation}
\hat{g}(\bx)= \begin{cases}
c_2-c_1+\gamma\left(\mathfrak{P}-2c_1 c_2^2\mathfrak{D}\right),& \text{if}\ \bx\in\mathcal{C}_1\\
c_2-c_1+\gamma\left(\mathfrak{P}+2c_1^2 c_2\mathfrak{D}\right),& \text{if}\ \bx\in\mathcal{C}_2
\end{cases}\label{eq:random-eq-g(x)}
\end{equation}
with
\begin{align}
	\mathfrak{P} &= -\frac{2\fftau}{n}\by^{\T}\bP\bOmega^{\T}\bomega_{\bx}-\frac{4c_1 c_2\fftau}{\sqrt{p}}\left(\bmu_2-\bmu_1\right)^{\T}\bomega_\bx\nonumber\\\label{eq:P}
	&+2c_1 c_2\ffftau\psi_\bx\frac{\tr\left(\bC_2-\bC_1\right)}{p}\\
	\mathbb{\mathfrak{D}} &=-\frac{2\fftau}{p}\|\bmu_{2}-\bmu_{1}\|^{2}+\frac{\ffftau}{p^{2}}\left(\tr\left({\bf C}_{2}-{\bf C}_{1}\right)\right)^{2}\nonumber \\
	& +\frac{2\ffftau}{p^{2}}\tr\left((\bC_{2}-\bC_{1})^{2}\right).\label{eq:D}
\end{align}

\end{theorem}
Leaving the proof to Appendix~A in the Supplementary Material, Theorem~\ref{thm:Random Equivalent} tells us that the decision function $g(\bx)$ has an asymptotic equivalent $\hat g(\bx)$ that consists of three parts:
\begin{enumerate}
\item the deterministic term $c_2-c_1$ of order $O(1)$ that depends on the number of instances in each class of the training set, which essentially comes from the term $\mathbf{1}_n^{\T}\by/n$ in $b$;
\item the ``noisy'' term $\mathfrak{P}$ of order $O(n^{-1})$ which is a function of the zero mean random variables $\bomega_\bx$ and $\psi_\bx$, thus in particular $\mathbb{E}[\mathfrak{P}]=0$;
\item the ``informative'' term containing $\mathfrak{D}$, also of order $O(n^{-1})$, which features the deterministic differences between the two classes.
\end{enumerate}

From Theorem~\ref{thm:Random Equivalent}, under the basic settings of Assumption~\ref{as:Growth rate}, for Gaussian data ${\bf x}\in\mathcal{C}_{a}$, $a\in\{1,2\}$, we can show that $\hat{g}(\bx)$ (and therefore $g(\bx)$) converges to a random Gaussian variable the mean and variance of which are given in the following theorem. The proof is deferred to Appendix~B.

\begin{theorem}[Gaussian Approximation]
\label{thm:Gaussian Approximation}
Under the setting of Theorem~\ref{thm:Random Equivalent}, $n(g({\bf x})-G_{a})\cd0$, where
\[
G_{a}\sim\mathcal{N}({\rm E}_{a},{\rm Var}_{a})
\]
with
\begin{align*}
&{\rm E}_{a} =\begin{cases}
c_{2}-c_{1}-2c_{2}\cdot c_{1}c_{2}\gamma\mathfrak{D}, & a=1\\
c_{2}-c_{1}+2c_{1}\cdot c_{1}c_{2}\gamma\mathfrak{D}, & a=2
\end{cases}\\
&{\rm Var}_{a}=8\gamma^{2}c_{1}^{2}c_{2}^{2}\left(\mathcal{V}_{1}^{a}+\mathcal{V}_{2}^{a}+\mathcal{V}_{3}^{a}\right)
\end{align*}
and
\begin{align}
\mathcal{V}_{1}^{a} & =\frac{\left(\ffftau\right)^{2}}{p^{4}}\left(\tr\left({\bf C}_{2}-{\bf C}_{1}\right)\right)^{2}\text{\ensuremath{\tr}}{\bf C}_{a}^{2}\nonumber \\
\mathcal{V}_{2}^{a} & =\frac{2\left(\fftau\right)^{2}}{p^{2}}\left(\bmu_{2}-\bmu_{1}\right)^{\T}\bC_{a}\left(\bmu_{2}-\bmu_{1}\right)\nonumber \\
\mathcal{V}_{3}^{a} & =\frac{2\left(\fftau\right)^{2}}{np^{2}}\left(\frac{\tr\bC_{1}\bC_{a}}{c_{1}}+\frac{\tr\bC_{2}\bC_{a}}{c_{2}}\right)\nonumber.
\end{align}
\end{theorem}

Theorem~\ref{thm:Gaussian Approximation} is our main practical result as it allows one to evaluate the large $n,p$ performance of LS-SVM for Gaussian data. While dwelling on the implications of Theorem~\ref{thm:Random Equivalent} and~\ref{thm:Gaussian Approximation}, several remarks and discussions are in order.

\begin{remark}[Dominant Bias]\label{rem:Dominant Bias}
\normalfont From Theorem~\ref{thm:Random Equivalent}, under the key Assumption~\ref{as:Growth rate}, both the random noise $\mathfrak{P}$ and the deterministic ``informative'' term $\mathfrak{D}$ are of order $O(n^{-1})$, which means that the decision function $g(\bx)=c_2-c_1+O(n^{-1})$. This result somehow contradicts the classical decision criterion proposed in \cite{suykens1999least}, based on the sign of $g(\bx)$, i.e., $\bx$ is associated to class $\mathcal{C}_{1}$ if $g({\bx})<0$ and to class $\mathcal{C}_{2}$ otherwise. When $c_1\neq c_2$, this would lead to an asymptotic classification of all new data $\bx$'s in the same class as $n\to\infty$. Practically speaking, this means for $n,p$ large that the decision function $g(\bx)$ of a new datum $\bx$ lies (sufficiently) away from $0$ ($0$ being the classically considered threshold), so that the sign of $g(\bx)$ is constantly positive (in the case of $\bar c_2 > \bar c_1$) or negative (in the case of $\bar c_2 < \bar c_1$). As such, all new data will be trivially classified into the same class. Instead, a first result of Theorem~\ref{thm:Random Equivalent} is that the decision threshold $\xi$ should be taken as $\xi=\xi_n=c_2-c_1+O(n^{-1})$ for imbalanced classification problem.
\end{remark}

The conclusion of Remark~\ref{rem:Dominant Bias} was in fact already known since the work of \cite{gestel2002bayesian} who reached the same conclusion through a Bayesian inference analysis, \emph{for all finite $n,p$}. From their Bayesian perspective, the term $c_2-c_1$ appears in the ``bias term'' $b$ under the form of prior class probabilities ${\rm P}(y=-1)$, ${\rm P}(y=1)$ and allows for adjusting classification problems with different prior class probabilities in the training and test sets. This idea of a (static) bias term correction has also been applied in \cite{evgeniou2000image} in order to improve the validation set performance. Here we confirm the problem of imbalanced datasets in Remark~\ref{rem:Dominant Bias} by Figure~\ref{fig:Gaussian-approximation-Gauss} with $c_1=1/4$ and $c_2=3/4$, where the histograms of $g(\bx)$ for $\bx\in\mathcal{C}_1$ and $\mathcal{C}_2$ center somewhere close to $c_2-c_1=0.5$, thus resulting in a trivial classification by assigning all new data to $\mathcal{C}_2$ if one takes $\xi=0$ because ${\rm P}(g(\bx)<\xi\mid\bx\in\mathcal{C}_1)\to0$ and ${\rm P}(g(\bx)>\xi\mid\bx\in\mathcal{C}_2)\to1$ as $n,p\to\infty$ (the convergence being in fact an equality for finite $n,p$ in this particular figure).
\medskip 

\begin{figure}[htb]
\noindent\begin{minipage}[b]{1\linewidth}%
\centering{}
\definecolor{mycolor1}{rgb}{0.00000,1.00000,1.00000}%
\definecolor{mycolor2}{rgb}{1.00000,0.00000,1.00000}%
\begin{tikzpicture}[font=\footnotesize]
\begin{axis}[%
width=7.5cm,
height=4cm,
scale only axis,
xmin=0.49,
xmax=0.51,
xtick={0.49,0.50,0.51},
every outer y axis line/.append style={white},
every x tick label/.append style={font=\tiny},
every y tick label/.append style={font=\color{white}},
ymin=0,
ymax=350,
ytick={\empty},
axis background/.style={fill=white},
axis x line*=bottom,
axis y line*=left,
legend style={legend cell align=left,align=left,draw=white!15!black}
]
\addplot[ybar,bar width=2.5,draw=white,fill=blue,area legend] plot table[row sep=crcr] {%
0.490045132537101	2.79980504476034\\
0.490379977279864	0\\
0.490714822022628	2.79980504476034\\
0.491049666765391	2.79980504476034\\
0.491384511508154	6.53287843777413\\
0.491719356250918	9.33268348253447\\
0.492054200993681	17.7320986168155\\
0.492389045736444	39.1972706266448\\
0.492723890479207	61.5957109847275\\
0.493058735221971	92.3935664770913\\
0.493393579964734	104.526055004386\\
0.493728424707497	154.922545810072\\
0.49406326945026	171.721376078634\\
0.494398114193024	218.384793491307\\
0.494732958935787	243.58303889415\\
0.49506780367855	277.180699431274\\
0.495402648421313	254.782259073191\\
0.495737493164077	258.515332466205\\
0.49607233790684	284.646846217301\\
0.496407182649603	221.184598536067\\
0.496742027392366	170.788107730381\\
0.49707687213513	153.056009113565\\
0.497411716877893	95.1933715218516\\
0.497746561620656	55.9961008952068\\
0.498081406363419	37.3307339301379\\
0.498416251106183	27.9980504476034\\
0.498751095848946	13.9990252238017\\
0.499085940591709	4.66634174126724\\
0.499420785334472	1.86653669650689\\
0.499755630077236	0.933268348253447\\
};
\addlegendentry{$g(\mathbf{x})_{\mathbf{x}\in\mathcal{C}_1}$};

\addplot[ybar,bar width=2.5,draw=white,fill=red,area legend] plot table[row sep=crcr] {%
0.495795463485059	0.279955901906341\\
0.496167545895029	0.279955901906341\\
0.496539628304999	1.3997795095317\\
0.496911710714969	3.91938262668877\\
0.497283793124939	5.03920623431413\\
0.497655875534909	10.9182801743473\\
0.498027957944879	19.3169572315375\\
0.498400040354849	36.3942672478243\\
0.498772122764819	59.6306071060505\\
0.49914420517479	91.265624021467\\
0.49951628758476	143.617377677953\\
0.49988836999473	183.651071650559\\
0.5002604524047	245.521325971861\\
0.50063253481467	279.395990102528\\
0.50100461722464	300.112726843597\\
0.50137669963461	289.19444666925\\
0.50174878204458	264.838283203398\\
0.50212086445455	237.682560718483\\
0.50249294686452	181.691380337215\\
0.502865029274491	131.859229797886\\
0.503237111684461	90.425756315748\\
0.503609194094431	59.6306071060505\\
0.503981276504401	26.315854779196\\
0.504353358914371	12.5980155857853\\
0.504725441324341	6.71894164575217\\
0.505097523734311	4.19933852859511\\
0.505469606144281	1.3997795095317\\
0.505841688554251	0\\
0.506213770964221	0\\
0.506585853374191	0.279955901906341\\
};
\addlegendentry{$g(\mathbf{x})_{\mathbf{x}\in\mathcal{C}_2}$};

\addplot [color=mycolor1,dashed,line width=2.0pt]
  table[row sep=crcr]{%
0.490045132537101	0.00942840125549684\\
0.490379977279864	0.0301529436110672\\
0.490714822022628	0.0900185401602058\\
0.491049666765391	0.250867734231119\\
0.491384511508154	0.652631814199151\\
0.491719356250918	1.58490146091898\\
0.492054200993681	3.59291449850242\\
0.492389045736444	7.60329936838636\\
0.492723890479207	15.0199293660171\\
0.493058735221971	27.6977366545642\\
0.493393579964734	47.679453049634\\
0.493728424707497	76.6176521494566\\
0.49406326945026	114.930952433694\\
0.494398114193024	160.936948916681\\
0.494732958935787	210.370632696009\\
0.49506780367855	256.699518598587\\
0.495402648421313	292.398800669681\\
0.495737493164077	310.911446491644\\
0.49607233790684	308.608886272031\\
0.496407182649603	285.950417398957\\
0.496742027392366	247.333893411994\\
0.49707687213513	199.704166039119\\
0.497411716877893	150.522425273951\\
0.497746561620656	105.907296335218\\
0.498081406363419	69.5602520346117\\
0.498416251106183	42.6488184744334\\
0.498751095848946	24.4097568315608\\
0.499085940591709	13.0415885203247\\
0.499420785334472	6.50441324668118\\
0.499755630077236	3.02828223848246\\
};
\addlegendentry{$G_1$};

\addplot [color=mycolor2,dashed,line width=2.0pt]
  table[row sep=crcr]{%
0.495795463485059	0.0419414968037223\\
0.496167545895029	0.131714450163446\\
0.496539628304999	0.382252400359953\\
0.496911710714969	1.02516629552499\\
0.497283793124939	2.54077213467114\\
0.497655875534909	5.81921517531605\\
0.498027957944879	12.3165880225125\\
0.498400040354849	24.0903835077414\\
0.498772122764819	43.5435981000682\\
0.49914420517479	72.7331190584764\\
0.49951628758476	112.270953609575\\
0.49988836999473	160.151094501065\\
0.5002604524047	211.115301002137\\
0.50063253481467	257.179774231741\\
0.50100461722464	289.521760438435\\
0.50137669963461	301.198573803569\\
0.50174878204458	289.568896815508\\
0.50212086445455	257.263522752857\\
0.50249294686452	211.21843139179\\
0.502865029274491	160.255415310759\\
0.503237111684461	112.362376262016\\
0.503609194094431	72.8041970618435\\
0.503981276504401	43.593246981294\\
0.504353358914371	24.1217781957943\\
0.504725441324341	12.3346469001044\\
0.505097523734311	5.82869625082939\\
0.505469606144281	2.54532607098359\\
0.505841688554251	1.02717094976386\\
0.506213770964221	0.383062228469196\\
0.506585853374191	0.13201498588245\\
};
\addlegendentry{$G_2$};
\end{axis}
\end{tikzpicture}%
\caption{Gaussian approximation of $g({\bf x})$, $n=256,\ p=512,\ c_{1}=1/4, c_{2}=3/4,\gamma=1$, Gaussian kernel with $\sigma^2=1$, $\bx\in\mathcal{N}(\bmu_a,\bC_a)$ with $\bmu_a=\left[\mathbf{0}_{a-1};3;\mathbf{0}_{p-a}\right]$, ${\bf C}_{1}={\bf I}_{p}$ and $\{{\bf C}_{2}\}_{i,j}=.4^{\mid i-j\mid}(1+5/\sqrt{p})$.\label{fig:Gaussian-approximation-Gauss}}
\end{minipage}
\end{figure}

An alternative to alleviate this imbalance issue is to normalize the label vector $\by$. From the proof of Theorem~\ref{thm:Random Equivalent} in Appendix~A we see the term $c_2-c_1$ is due to the fact that in $b$ one has $\mathbf{1}_n^{\T}\by/n=c_2-c_1\neq 0$. Thus, one may normalize the labels $y_i$ as $y^*_i = -1/c_1$ if $\bx_i\in\mathcal{C}_1$ and $y^*_i = 1/c_2$ if $\bx_i\in\mathcal{C}_2$, so that the relation $\mathbf{1}_n^{\T}\by^*= 0$ is satisfied. This formulation is also referred to as the \emph{Fisher’s targets}: $\{-n/n_1,n/n_2\}$ in the context of kernel fisher discriminant analysis \cite{baudat2000generalized,mika1999fisher}. With the aforementioned normalized labels $\by^*$, we have the following lemma that reveals the connection between the corresponding decision function $g^*(\bx)$ and $g(\bx)$.

\begin{lemma}\label{lem:norml-y}
Let $g(\bx)$ be defined by \eqref{eq:decision-function} and $g^*(\bx)$ be defined as $g^*(\bx)=(\balpha^*)^{\T} \bk(\bx)+b^*$, with $(\balpha^*,b^*)$ given by \eqref{eq:LS-SVM-solution-2} for $\by^*$ in the place of $\by$, where $y^*_i = -1/c_1$ if $\bx_i\in\mathcal{C}_1$ and $y^*_i = 1/c_2$ if $\bx_i\in\mathcal{C}_2$. Then, 
\begin{equation*}
g(\bx)-(c_2-c_1)=2c_1c_2g^*(\bx).
\end{equation*}
\end{lemma}

\begin{IEEEproof}
From \eqref{eq:LS-SVM-solution-2} and \eqref{eq:decision-function} we get
\begin{align*}
g(x)&=\by^{\T}\left(\bS^{-1}-\frac{\bS^{-1}\mathbf{1}_n\mathbf{1}_n^{\T}\bS^{-1}}{\mathbf{1}_n^{\T}\bS^{-1}\mathbf{1}_n}\right)\bk(\bx)+\frac{\by^{\T}\bS^{-1}\mathbf{1}_n}{\mathbf{1}_n^{\T}\bS^{-1}\mathbf{1}_n}=\by^{\T}\bvarpi
\end{align*}
with $\bvarpi = \left(\bS^{-1}-\frac{\bS^{-1}\mathbf{1}_n\mathbf{1}_n^{\T}\bS^{-1}}{\mathbf{1}_n^{\T}\bS^{-1}\mathbf{1}_n}\right)\bk(\bx)+\frac{\bS^{-1}\mathbf{1}_n}{\mathbf{1}_n^{\T}\bS^{-1}\mathbf{1}_n}$.
Besides, note that $\mathbf{1}_n^{\T}\bvarpi=1$. We thus have
\begin{align*}
g(\bx)-(c_2-c_1)&=\by^{\T}\bvarpi-(c_2-c_1)\mathbf{1}_n^{\T}\bvarpi\\
&=2c_1c_2\left(\frac{\by-(c_2-c_1)\mathbf{1}_n}{2c_1 c_2}\right)^{\T}\bvarpi\\
&=2c_1c_2(\by^*)^{\T}\bvarpi=2c_1c_2g^*(\bx)
\end{align*}
which concludes the proof.
\end{IEEEproof}

As a consequence of Lemma~\ref{lem:norml-y}, instead of Theorem~\ref{thm:Gaussian Approximation} for standard labels $\by$, one would have the following corollary for the corresponding Gaussian approximation of $g^*(\bx)$ when normalized labels $\by^*$ are used.

\begin{corollary}[Gaussian Approximation of $g^*(\bx)$]
\label{cor:Gaussian Approximation normal}
Under the setting of Theorem~\ref{thm:Random Equivalent}, and with $g^*(\bx)$ defined in Lemma~\ref{lem:norml-y}, $n(g^*({\bf x})-G^*_{a})\cd0$, where
\[
G^*_{a}\sim\mathcal{N}({\rm E}^*_a,{\rm Var}^*_a)
\]
with
\begin{align*}
&{\rm E^*_a} =\begin{cases}
-c_{2}\gamma\mathfrak{D}, & a=1\\
+c_{1}\gamma\mathfrak{D}, & a=2
\end{cases}\\
&{\rm Var^*_a}=2\gamma^{2}\left(\mathcal{V}_{1}^{a}+\mathcal{V}_{2}^{a}+\mathcal{V}_{3}^{a}\right)
\end{align*}
and $\mathfrak{D}$ is defined by \eqref{eq:D}, $\mathcal{V}_{1}^{a},\mathcal{V}_{2}^{a}$ and $\mathcal{V}_{3}^{a}$ as in Theorem~\ref{thm:Gaussian Approximation}.
\end{corollary}

Figure~\ref{fig:Gaussian-approximation-Gauss-star} illustrates this result in the same settings as Figure~\ref{fig:Gaussian-approximation-Gauss}. Compared to Figure~\ref{fig:Gaussian-approximation-Gauss}, one can observe that in Figure~\ref{fig:Gaussian-approximation-Gauss-star} both histograms are now centered close to $0$ (at distance $O(n^{-1})$ from zero) instead of $c_2-c_1=1/2$. Still, even in the case where normalized labels $\by^*$ are used as observed in Figure~\ref{fig:Gaussian-approximation-Gauss-star} (where the histograms cross at about $-0.004\approx 1/n$), taking $\xi=0$ as a decision threshold may not be an appropriate choice, as ${\rm E}^*_1\neq-{\rm E}^*_2$. 

\begin{figure}[htb]
\noindent\begin{minipage}[b]{1\linewidth}%
\centering{}
\definecolor{mycolor1}{rgb}{0.00000,1.00000,1.00000}%
\definecolor{mycolor2}{rgb}{1.00000,0.00000,1.00000}%
\begin{tikzpicture}[font=\footnotesize]
\begin{axis}[%
width=7.5cm,
height=4cm,
scale only axis,
scaled x ticks = false,
xmin=-0.025,
xmax=0.025,
x tick label style={/pgf/number format/fixed},
every outer y axis line/.append style={white},
every x tick label/.append style={font=\tiny},
every y tick label/.append style={font=\color{white}},
ymin=0,
ymax=120,
ytick={\empty},
axis background/.style={fill=white},
axis x line*=bottom,
axis y line*=left,
legend style={legend cell align=left,align=left,draw=white!15!black}
]
\addplot[ybar,bar width=2.5,draw=white,fill=blue,area legend] plot table[row sep=crcr] {%
-0.0253701625193252	1.02198330859184\\
-0.0244528285549459	1.36264441145579\\
-0.0235354945905666	1.36264441145579\\
-0.0226181606261874	3.06594992577553\\
-0.0217008266618081	5.45057764582317\\
-0.0207834926974288	7.15388316014291\\
-0.0198661587330495	16.0110718346056\\
-0.0189488247686703	22.8242938918845\\
-0.018031490804291	36.4507380064425\\
-0.0171141568399117	42.2419767551296\\
-0.0161968228755325	73.5827982186128\\
-0.0152794889111532	88.2312256417626\\
-0.0143621549467739	92.6598199789939\\
-0.0134448209823946	97.4290754190892\\
-0.0125274870180154	101.176347550593\\
-0.0116101530536361	106.286264093552\\
-0.0106928190892568	94.0224643904497\\
-0.00977548512487756	80.3960202758918\\
-0.00885815116049829	64.3849484412862\\
-0.00794081719611901	50.7585043267283\\
-0.00702348323173974	34.0661102863948\\
-0.00610614926736047	30.3188381548914\\
-0.0051888153029812	16.6923940403335\\
-0.00427148133860192	11.9231386002382\\
-0.00335414737422266	5.45057764582317\\
-0.00243681340984338	2.72528882291159\\
-0.00151947944546411	1.70330551431974\\
-0.000602145481084844	0.681322205727896\\
0.00031518848329443	0.340661102863948\\
0.0012325224476737	0.340661102863948\\
};
\addlegendentry{$g^{*}(\mathbf{x})_{\mathbf{x}\in\mathcal{C}_1}$};

\addplot[ybar,bar width=2.5,draw=white,fill=red,area legend] plot table[row sep=crcr] {%
-0.0102820384337524	0.340240516182223\\
-0.00936357051139414	0\\
-0.00844510258903584	0.680481032364446\\
-0.00752663466667755	0.567067526970371\\
-0.00660816674431925	2.72192412945778\\
-0.00568969882196096	5.10360774273334\\
-0.00477123089960266	9.75356146389039\\
-0.00385276297724437	19.2802959169926\\
-0.00293429505488607	25.8582792298489\\
-0.00201582713252778	40.4886214256845\\
-0.00109735921016948	54.2116555783675\\
-0.000178891287811189	71.5639219036609\\
0.000739576634547108	93.3393149393231\\
0.0016580445569054	104.11359795176\\
0.0025765124792637	108.990378683705\\
0.00349498040162199	114.434226942621\\
0.00441344832398029	105.361146511095\\
0.00533191624633858	93.3393149393231\\
0.00625038416869688	76.7809431517883\\
0.00716885209105517	52.2836259866682\\
0.00808732001341346	43.3239590605364\\
0.00900578793577176	29.4875114024593\\
0.00992425585813005	19.0534689062045\\
0.0108427237804883	8.61942640994964\\
0.0117611917028466	4.30971320497482\\
0.0126796596252049	2.60851062406371\\
0.0135981275475632	1.24754855933482\\
0.0145165954699215	0.453654021576297\\
0.0154350633922798	0.340240516182223\\
0.0163535313146381	0.113413505394074\\
};
\addlegendentry{$g^{*}(\mathbf{x})_{\mathbf{x}\in\mathcal{C}_2}$};

\addplot [color=mycolor1,dashed,line width=2.0pt]
  table[row sep=crcr]{%
-0.0253701625193252	0.0145490983471192\\
-0.0244528285549459	0.0441721042337502\\
-0.0235354945905666	0.124652803143097\\
-0.0226181606261874	0.326962512676314\\
-0.0217008266618081	0.797142248666319\\
-0.0207834926974288	1.80640715489481\\
-0.0198661587330495	3.80484889747277\\
-0.0189488247686703	7.44905422811722\\
-0.018031490804291	13.5552274644427\\
-0.0171141568399117	22.9273803469001\\
-0.0161968228755325	36.0449156649161\\
-0.0152794889111532	52.6714827099757\\
-0.0143621549467739	71.540014919734\\
-0.0134448209823946	90.3159513455419\\
-0.0125274870180154	105.97949546544\\
-0.0116101530536361	115.590245601127\\
-0.0106928190892568	117.182424017626\\
-0.00977548512487756	110.419485871914\\
-0.00885815116049829	96.7098959399813\\
-0.00794081719611901	78.7296033568901\\
-0.00702348323173974	59.5726834153005\\
-0.00610614926736047	41.8984748862734\\
-0.0051888153029812	27.3899486947477\\
-0.00427148133860192	16.642791593296\\
-0.00335414737422266	9.39946387421959\\
-0.00243681340984338	4.93425901701214\\
-0.00151947944546411	2.4075912482116\\
-0.000602145481084844	1.0919066743198\\
0.00031518848329443	0.460288619761989\\
0.0012325224476737	0.180350315523425\\
};
\addlegendentry{$G^*_1$};

\addplot [color=mycolor2,dashed,line width=2.0pt]
  table[row sep=crcr]{%
-0.0102820384337524	0.0444865882670874\\
-0.00936357051139414	0.120842796030352\\
-0.00844510258903584	0.306653695892876\\
-0.00752663466667755	0.726961505486405\\
-0.00660816674431925	1.60994262396801\\
-0.00568969882196096	3.33077392436521\\
-0.00477123089960266	6.43747694252642\\
-0.00385276297724437	11.6231004232317\\
-0.00293429505488607	19.6048745136837\\
-0.00201582713252778	30.891707686107\\
-0.00109735921016948	45.4732023601696\\
-0.000178891287811189	62.5323767337015\\
0.000739576634547108	80.332282133744\\
0.0016580445569054	96.4075501950749\\
0.0025765124792637	108.085581301341\\
0.00349498040162199	113.203603396488\\
0.00441344832398029	110.761419463346\\
0.00533191624633858	101.240095834812\\
0.00625038416869688	86.4474835162635\\
0.00716885209105517	68.9585207148374\\
0.00808732001341346	51.3877130858935\\
0.00900578793577176	35.7739106466447\\
0.00992425585813005	23.2653354240889\\
0.0108427237804883	14.1347454502848\\
0.0117611917028466	8.02236538509986\\
0.0126796596252049	4.25356092898262\\
0.0135981275475632	2.10687442320315\\
0.0145165954699215	0.9749007779243\\
0.0154350633922798	0.421422729868403\\
0.0163535313146381	0.170181077173769\\
};
\addlegendentry{$G^*_2$};
\end{axis}
\end{tikzpicture}%
\caption{Gaussian approximation of $g^*({\bf x})$, $n=256,\ p=512,\ c_{1}=1/4, c_{2}=3/4,\gamma=1$, Gaussian kernel with $\sigma^2=1$, $\bx\in\mathcal{N}(\bmu_a,\bC_a)$,  with $\bmu_a=\left[\mathbf{0}_{a-1};3;\mathbf{0}_{p-a}\right]$, ${\bf C}_{1}={\bf I}_{p}$ and $\{{\bf C}_{2}\}_{i,j}=.4^{\mid i-j\mid}(1+5/\sqrt{p})$.\label{fig:Gaussian-approximation-Gauss-star}}
\end{minipage}
\end{figure}

\begin{remark}[Insignificance of $\gamma$]\label{rem:Insignificance}
\normalfont As a direct result of Theorem~\ref{thm:Random Equivalent} and Remark~\ref{rem:Dominant Bias}, note in  \eqref{eq:random-eq-g(x)} that $\hat{g}(\bx)-\left(c_2-c_1\right)$ is proportional to the hyperparameter $\gamma$, which indicates that, rather surprisingly, the tuning of $\gamma$ is (asymptotically) of no importance when $n,p\to\infty$ since it does not alter the classification statistics when one uses the sign of $g(\bx)-\left(c_2-c_1\right)$ for the decision.\footnote{This remark is only valid only under Assumption~\ref{as:Growth rate} and $\gamma=O(1)$, i.e., $\gamma$ is considered to remain a constant as $n,p\to\infty$. Recall that this is in sharp contrast with \citep{caponnetto2007optimal} where $\gamma=O(\sqrt{n})$ (or $O(n)$, depending on the problem) is claimed optimal in the large $n$ only regime. From Remark~\ref{rem:optimal-growth-rate} on the growth rate optimality reached by LS-SVM, we see here that $\gamma=O(1)$ is rate-optimal under the present large $n,p$ setting; yet we believe that more elaborate kernels (such as those explored in \cite{cheng2013spectrum}) may allow for improved performances (not in the rate but in the constants), possibly for different scales of $\gamma$. This intuition will be explored in future investigations.}
\end{remark}


Letting $Q(x)=\frac{1}{2\pi}\int_{x}^{\infty}\exp\left(-t^2/2\right)dt$, from Theorem~\ref{thm:Gaussian Approximation} and Corollary~\ref{cor:Gaussian Approximation normal}, we now have the following immediate corollary for the (asymptotic) classification error rate.

\begin{corollary}[Asymptotic Error Rate]\label{cor:Asymptotic Error Rate}
Under the setting of Theorem~\ref{thm:Random Equivalent}, for a threshold $\xi_{n}$ possibly depending on $n$, as $n\to\infty$,
\begin{align}
&{\rm P}(g({\bf x})>\xi_{n}~|~{\bf x}\in\mathcal{C}_{1})-Q\left(\frac{\xi_{n}-{\rm {E}_1}}{\sqrt{{\rm {Var}_{1}}}}\right)\to0\label{eq:proba-1}\\
&{\rm P}(g({\bf x})<\xi_{n}~|~{\bf x}\in\mathcal{C}_{2})-Q\left(\frac{{\rm {E}_{2}}-\xi_{n}}{\sqrt{{\rm {Var}_{2}}}}\right)\to0\label{eq:proba-2}
\end{align}
with ${\rm E}_a$ and ${\rm Var}_a$ given in Theorem~\ref{thm:Gaussian Approximation}.
\end{corollary}

Obviously, Corollary~\ref{cor:Asymptotic Error Rate} is only meaningful when $\xi_n=c_2-c_1+O(n^{-1})$ as recalled earlier. Besides, it is clear from Lemma~\ref{lem:norml-y} and Corollary~\ref{cor:Gaussian Approximation normal} that ${\rm P}(g(\bx)>\xi_n~|~\bx\in\mathcal{C}_a) = {\rm P}(g^*(\bx)>\xi_n-(c_2 - c_1)~|~\bx\in\mathcal{C}_a)$, so that Corollary~\ref{cor:Asymptotic Error Rate} extends naturally to $g^*(\bx)$ when normalized labels $\by^*$ are applied.

Corollary~\ref{cor:Asymptotic Error Rate} allows one to compute the asymptotic misclassification rate as a function of $\rm{E}_a, \rm{Var}_a$ and the threshold $\xi_n$. Combined with Theorem~\ref{thm:Gaussian Approximation}, one may note the significance of a proper choice of the kernel function $f$. For instance, if $\fftau=0$, the term $\bmu_{2}-\bmu_{1}$ vanishes from the mean and variance of $G_{a}$, meaning that the classification of LS-SVM will not rely (at least asymptotically and under Assumption~\ref{as:Growth rate}) on the differences in means of the two classes. Figure~\ref{fig:Gaussian-approximation-Gauss-dif-mean-2} corroborates this finding with the same theoretical Gaussian approximations $G_1$ and $G_2$ in subfigures (a) and (b). When $\|\bmu_2-\bmu_1\|^2$ varies from $0$ in (a) to $18$ in (b), the distribution of $g(\bx)$, and in particular, the overlap between two classes, remain almost the same in (a) and (b).

\begin{figure}[htb]
\begin{minipage}[b]{0.48\linewidth}%
\centering{}
\definecolor{mycolor1}{rgb}{0.00000,1.00000,1.00000}%
\definecolor{mycolor2}{rgb}{1.00000,0.00000,1.00000}%
\begin{tikzpicture}[font=\footnotesize]
\begin{axis}[%
width=4.5cm,
height=3cm,
scale only axis,
scaled x ticks = false,
xmin=-0.08,
xmax=0.08,
x tick label style={/pgf/number format/fixed},
every outer y axis line/.append style={white},
every y tick label/.append style={font=\color{white}},
ymin=0,
ymax=30,
ytick={\empty},
axis background/.style={fill=white},
axis x line*=bottom,
axis y line*=left,
]
\addplot[ybar,bar width=2,draw=white,fill=blue,area legend] plot table[row sep=crcr] {%
-0.0740541244783661	0.0934886610034484\\
-0.0707114731167377	0.140232991505173\\
-0.0673688217551092	0.233721652508621\\
-0.0640261703934807	0.607676296522414\\
-0.0606835190318522	0.981630940536208\\
-0.0573408676702238	2.15023920307931\\
-0.0539982163085953	2.57093817759483\\
-0.0506555649469668	4.72117738067414\\
-0.0473129135853384	7.99328051579484\\
-0.0439702622237099	11.7795712864345\\
-0.0406276108620814	16.6409816586138\\
-0.0372849595004529	20.9414600647724\\
-0.0339423081388245	24.7744951659138\\
-0.030599656777196	28.1868312925397\\
-0.0272570054155675	28.0933426315362\\
-0.023914354053939	29.9631158516052\\
-0.0205717026923106	27.9998539705328\\
-0.0172290513306821	23.9330972168828\\
-0.0138863999690536	18.5574992091845\\
-0.0105437486074252	16.9214476416242\\
-0.00720109724579668	11.5925939644276\\
-0.00385844588416821	7.99328051579484\\
-0.00051579452253972	5.32885367719656\\
0.00282685683908875	2.7111711691\\
0.00616950820071721	2.01000621157414\\
0.00951215956234569	1.07511960153966\\
0.0128548109239742	0.607676296522414\\
0.0161974622856026	0.420698974515518\\
0.0195401136472311	0.0934886610034484\\
0.0228827650088596	0.0467443305017242\\
};
\addplot[ybar,bar width=2,draw=white,fill=red,area legend] plot table[row sep=crcr] {%
-0.0432150644190245	0.0540835726831453\\
-0.0374369698055119	0.135208931707863\\
-0.0316588751919992	0.40562679512359\\
-0.0258807805784866	0.730128231222461\\
-0.020102685964974	1.40617288976178\\
-0.0143245913514614	3.9481008058696\\
-0.00854649673794875	6.76044658539316\\
-0.00276840212443613	9.05899842442683\\
0.00300969248907649	13.9265199659099\\
0.0087877871025891	17.9287043444627\\
0.0145658817161017	20.4165486878873\\
0.0203439763296143	20.5517576195952\\
0.026122070943127	19.8216293883727\\
0.0319001655566396	17.0633671815323\\
0.0376782601701522	12.9800574439549\\
0.0434563547836648	9.95137737369873\\
0.0492344493971775	7.08494802149203\\
0.0550125440106901	4.73231260977521\\
0.0607906386242027	2.94755471123142\\
0.0665687332377153	1.51434003512807\\
0.0723468278512279	0.757170017564034\\
0.0781249224647406	0.51379394048988\\
0.0839030170782532	0.189292504391008\\
0.0896811116917658	0.108167145366291\\
0.0954592063052784	0.0270417863415726\\
0.101237300918791	0\\
0.107015395532304	0\\
0.112793490145816	0\\
0.118571584759329	0.0270417863415726\\
0.124349679372841	0.0270417863415726\\
};
\addplot [color=mycolor1,dashed,line width=2.0pt,forget plot]
  table[row sep=crcr]{%
-0.0740541244783661	0.061595274468184\\
-0.0707114731167377	0.139795124799252\\
-0.0673688217551092	0.299223382307644\\
-0.0640261703934807	0.604029152378387\\
-0.0606835190318522	1.14995048739393\\
-0.0573408676702238	2.06471091642689\\
-0.0539982163085953	3.49621584341497\\
-0.0506555649469668	5.58336589366985\\
-0.0473129135853384	8.40916380552913\\
-0.0439702622237099	11.944511850041\\
-0.0406276108620814	16.0008431922734\\
-0.0372849595004529	20.2151143775165\\
-0.0339423081388245	24.0862065845117\\
-0.030599656777196	27.0657135510688\\
-0.0272570054155675	28.6833205009854\\
-0.023914354053939	28.6680557508515\\
-0.0205717026923106	27.0225248775207\\
-0.0172290513306821	24.0221835229793\\
-0.0138863999690536	20.1399276457218\\
-0.0105437486074252	15.9243678689346\\
-0.00720109724579668	11.8747744015096\\
-0.00385844588416821	8.35117152486232\\
-0.00051579452253972	5.53896103683455\\
0.00282685683908875	3.4647195515234\\
0.00616950820071721	2.04393336779201\\
0.00951215956234569	1.13716700687463\\
0.0128548109239742	0.596678841984973\\
0.0161974622856026	0.295267669195143\\
0.0195401136472311	0.137800255765647\\
0.0228827650088596	0.0606517060850317\\
};
\addplot [color=mycolor2,dashed,line width=2.0pt,forget plot]
  table[row sep=crcr]{%
-0.0432150644190245	0.0481213949644315\\
-0.0374369698055119	0.127064956079454\\
-0.0316588751919992	0.308150946829026\\
-0.0258807805784866	0.686359036313111\\
-0.020102685964974	1.40407187023406\\
-0.0143245913514614	2.63801617987304\\
-0.00854649673794875	4.55214099030945\\
-0.00276840212443613	7.21446428837619\\
0.00300969248907649	10.5012888591546\\
0.0087877871025891	14.0388409908869\\
0.0145658817161017	17.237331699575\\
0.0203439763296143	19.438330241603\\
0.026122070943127	20.1325138694854\\
0.0319001655566396	19.1508119264919\\
0.0376782601701522	16.7311775919365\\
0.0434563547836648	13.4250514978356\\
0.0492344493971775	9.89362727571935\\
0.0550125440106901	6.69645975885163\\
0.0607906386242027	4.16279587120476\\
0.0665687332377153	2.3767040181778\\
0.0723468278512279	1.24627883366922\\
0.0781249224647406	0.600213102127211\\
0.0839030170782532	0.265488587703915\\
0.0896811116917658	0.107854029802473\\
0.0954592063052784	0.0402417658044803\\
0.101237300918791	0.013790112701722\\
0.107015395532304	0.00434018993303076\\
0.112793490145816	0.00125458412701451\\
0.118571584759329	0.000333074186084342\\
0.124349679372841	8.12142607588744e-05\\
};
\end{axis}
\end{tikzpicture}
\centerline{(a) $\mu_{\rm dif}=0$}\medskip{}
\end{minipage}\hfill{}
\begin{minipage}[b]{0.48\linewidth}
\centering{}
\definecolor{mycolor1}{rgb}{0.00000,1.00000,1.00000}%
\definecolor{mycolor2}{rgb}{1.00000,0.00000,1.00000}%
\begin{tikzpicture}[font=\footnotesize]
\begin{axis}[%
width=4.5cm,
height=3cm,
scale only axis,
scaled x ticks = false,
xmin=-0.08,
xmax=0.08,
x tick label style={/pgf/number format/fixed},
every outer y axis line/.append style={white},
every y tick label/.append style={font=\color{white}},
ymin=0,
ymax=30,
ytick={\empty},
axis background/.style={fill=white},
axis x line*=bottom,
axis y line*=left,
]
\addplot[ybar,bar width=2,draw=white,fill=blue,area legend] plot table[row sep=crcr] {%
-0.0817656840589771	0.037872784371349\\
-0.077640030173381	0.151491137485396\\
-0.0735143762877849	0.454473412456188\\
-0.0693887224021887	0.416600628084839\\
-0.0652630685165926	1.43916580611126\\
-0.0611374146309965	2.84045882785118\\
-0.0570117607454004	3.59791451527816\\
-0.0528861068598042	6.32475499001529\\
-0.0487604529742081	9.05159546475242\\
-0.044634799088612	12.3465277050598\\
-0.0405091452030159	17.0427529671071\\
-0.0363834913174198	20.943649757356\\
-0.0322578374318236	24.3522003507774\\
-0.0281321835462275	25.4883838819179\\
-0.0240065296606314	25.8671117256314\\
-0.0198808757750353	22.8751617602948\\
-0.0157552218894391	19.2393744606453\\
-0.011629568003843	15.9065694359666\\
-0.0075039141182469	12.91461947063\\
-0.00337826023265077	7.57455687426981\\
0.00074739365294535	5.03708032138942\\
0.00487304753854148	3.56004173090681\\
0.0089987014241376	1.62852972796801\\
0.0131243553097337	1.62852972796801\\
0.0172500091953298	0.75745568742698\\
0.021375663080926	0.530218981198886\\
0.0255013169665221	0.265109490599443\\
0.0296269708521182	0.037872784371349\\
0.0337526247377143	0.037872784371349\\
0.0378782786233105	0.037872784371349\\
};
\addplot[ybar,bar width=2,draw=white,fill=red,area legend] plot table[row sep=crcr] {%
-0.0343810928506831	0.0350923194108663\\
-0.0299285515931332	0.105276958232599\\
-0.0254760103355834	0.175461597054331\\
-0.0210234690780335	0.386015513519529\\
-0.0165709278204836	1.26332349879119\\
-0.0121183865629338	2.24590844229544\\
-0.00766584530538391	4.63218616223435\\
-0.00321330404783405	7.12374084040585\\
0.00123923720971582	11.2997268502989\\
0.00569177846726569	16.002097651355\\
0.0101443197248155	20.8097454106437\\
0.0145968609823654	23.4065770470478\\
0.0190494022399153	24.0733311158543\\
0.0235019434974651	23.6522232829239\\
0.027954484755015	21.792330354148\\
0.0324070260125649	20.6342838135894\\
0.0368595672701147	13.8614661672922\\
0.0413121085276646	10.6680651009033\\
0.0457646497852145	8.38706433919704\\
0.0502171910427643	5.19366327280821\\
0.0546697323003142	3.33377034403229\\
0.0591222735578641	2.14063148406284\\
0.0635748148154139	1.57915437348898\\
0.0680273560729638	0.666754068806459\\
0.0724798973305137	0.386015513519529\\
0.0769324385880635	0.350923194108663\\
0.0813849798456134	0.210553916465198\\
0.0858375211031633	0.0350923194108663\\
0.0902900623607131	0.105276958232599\\
0.094742603618263	0.0350923194108663\\
};
\addplot [color=mycolor1,dashed,line width=2.0pt,forget plot]
  table[row sep=crcr]{%
-0.0817656840589771	0.00743578344113974\\
-0.077640030173381	0.0239559842269916\\
-0.0735143762877849	0.0705903324026587\\
-0.0693887224021887	0.190248113233924\\
-0.0652630685165926	0.468963900007966\\
-0.0611374146309965	1.05731012269561\\
-0.0570117607454004	2.18026503724913\\
-0.0528861068598042	4.11206635195841\\
-0.0487604529742081	7.09340668736278\\
-0.044634799088612	11.1916347798405\\
-0.0405091452030159	16.1501336776876\\
-0.0363834913174198	21.3158490729199\\
-0.0322578374318236	25.7319722662648\\
-0.0281321835462275	28.4110598958693\\
-0.0240065296606314	28.6910022630595\\
-0.0198808757750353	26.5001252570712\\
-0.0157552218894391	22.3869050169668\\
-0.011629568003843	17.2975306791289\\
-0.0075039141182469	12.2241335341688\\
-0.00337826023265077	7.90125329114488\\
0.00074739365294535	4.67108538482579\\
0.00487304753854148	2.52571038982549\\
0.0089987014241376	1.24908852451721\\
0.0131243553097337	0.564997868153699\\
0.0172500091953298	0.233746072698765\\
0.021375663080926	0.0884475402096085\\
0.0255013169665221	0.0306105490131035\\
0.0296269708521182	0.00968947845617256\\
0.0337526247377143	0.00280526315039356\\
0.0378782786233105	0.000742832224367437\\
};
\addplot [color=mycolor2,dashed,line width=2.0pt,forget plot]
  table[row sep=crcr]{%
-0.0343810928506831	0.205166489486586\\
-0.0299285515931332	0.395174344588207\\
-0.0254760103355834	0.723652223891214\\
-0.0210234690780335	1.25988210032641\\
-0.0165709278204836	2.08539705406684\\
-0.0121183865629338	3.28175723129121\\
-0.00766584530538391	4.91001694928946\\
-0.00321330404783405	6.98422749364487\\
0.00123923720971582	9.44523152346865\\
0.00569177846726569	12.1441097684721\\
0.0101443197248155	14.8449112495739\\
0.0145968609823654	17.2523544510901\\
0.0190494022399153	19.0624178068732\\
0.0235019434974651	20.0247194975486\\
0.027954484755015	19.9992515838893\\
0.0324070260125649	18.9897781635076\\
0.0368595672701147	17.1429233895012\\
0.0413121085276646	14.71325366379\\
0.0457646497852145	12.0058083369505\\
0.0502171910427643	9.3139292488499\\
0.0546697323003142	6.86962939949062\\
0.0591222735578641	4.81717615234356\\
0.0635748148154139	3.21151971274494\\
0.0680273560729638	2.03557684373364\\
0.0724798973305137	1.22665732540066\\
0.0769324385880635	0.702777513177381\\
0.0813849798456134	0.382799447533185\\
0.0858375211031633	0.198236487063651\\
0.0902900623607131	0.0976011037236125\\
0.094742603618263	0.0456861656491742\\
};
\end{axis}
\end{tikzpicture}
\centerline{(b) $\mu_{\rm dif}=3$}\medskip{}
\end{minipage}
\caption{Gaussian approximation of $g({\bf x})$, $n=256,\ p=512,\ c_{1}=c_{2}=1/2,\gamma=1$, polynomial kernel with $\ftau=4,\fftau=0$, and $\ffftau=2$. $\bx\in\mathcal{N}(\bmu_a,\bC_a)$, with $\bmu_a=\left[\mathbf{0}_{a-1};\mu_{\rm dif};\mathbf{0}_{p-a}\right]$, ${\bf C}_{1}={\bf I}_{p}$ and $\{{\bf C}_{2}\}_{i,j}=.4^{\mid i-j\mid}(1+5/\sqrt{p})$.\label{fig:Gaussian-approximation-Gauss-dif-mean-2}}
\end{figure}
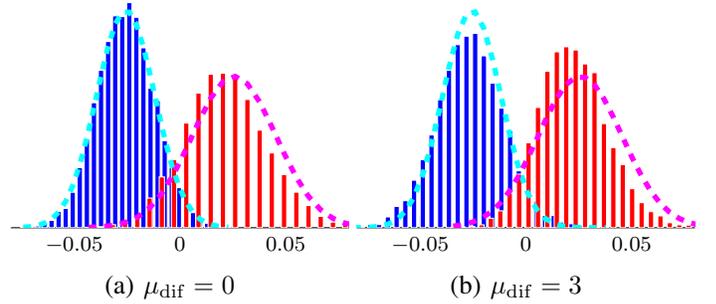

\medskip
More traceable special cases and discussions on the choice of kernel function $f$ will be given in the next section.

\section{Special cases and further discussions}
\label{sec:special}
\subsection{More discussions on the kernel function \texorpdfstring{$f$}{f}}
Following the discussion at the end of Section~\ref{sec:main}, if $\fftau=0$, the information about the statistical means of the two different classes is lost and will not help perform the classification. Nonetheless, we find that, rather surprisingly, if one further assumes $\tr\bC_{1}=\tr\bC_{2}+o(\sqrt{p})$ (which is beyond the minimum ``distance'' rate in Assumption~\ref{as:Growth rate}), using a kernel $f$ that satisfies $\fftau=0$ results in ${\rm Var}_{a}=0$ while ${\rm E}_{a}$ may remain non-zero, thereby ensuring a vanishing misclassification rate (as long as $\ffftau \neq 0$). Intuitively speaking, the kernels with $\fftau=0$ play an important role in extracting the information of ``shape'' of both classes, making the classification extremely accurate even in cases that are deemed impossible to classify according to Remark~\ref{rem:optimal-growth-rate}. This phenomenon was also remarked in \cite{couillet2016kernel} and deeply investigated in \cite{couillet2016random}. Figure \ref{fig:orth} substantiates this finding for $\bmu_1=\bmu_2$, ${\bf C}_{1}={\bf I}_{p}$ and $\{{\bf C}_{2}\}_{i,j}=.4^{\mid i-j\mid}$, for which $\tr\bC_1=\tr\bC_2=p$. We observe a rapid drop of the classification error as $\fftau$ gets close to $0$.

\begin{figure}[htb]
\noindent\begin{minipage}[b]{1\linewidth}%
\centering{}
\begin{tikzpicture}[font=\footnotesize]
\pgfplotsset{every major grid/.append style={densely dashed}}
/xlabel near ticks
/ylabel near ticks
\begin{axis}[%
width=7.5cm,
height=4.5cm,
scale only axis,
xmin=-3,
xmax=3,
xlabel={$f^{\prime}(\tau)$},
ymin=0,
ymax=0.65,
ylabel={Classification error\footnotemark},
grid=major,
axis background/.style={fill=white},
legend style={at={(0.98,0.98)},anchor=north east,legend cell align=left,align=left,draw=white!15!black}
]
\addplot [color=blue,smooth,solid,line width=1.0pt,mark size=2.0pt,mark=square,mark options={solid}]
  table[row sep=crcr]{%
-3	0.454736328125\\
-2.5	0.440234375\\
-2	0.424755859375\\
-1.5	0.39921875\\
-1	0.36103515625\\
-0.5	0.2416015625\\
0	0.084765625\\
0.5	0.25517578125\\
1	0.358447265625\\
1.5	0.401220703125\\
2	0.428662109375\\
2.5	0.440283203125\\
3	0.450439453125\\
};
\addlegendentry{Empirical error for $p=512$};
\addplot [color=blue,smooth,dashed,line width=1.0pt,mark size=2.0pt,mark=o,mark options={solid}]
  table[row sep=crcr]{%
-3	0.4483154296875\\
-2.5	0.4451904296875\\
-2	0.4262939453125\\
-1.5	0.40947265625\\
-1	0.3548583984375\\
-0.5	0.2361328125\\
0	0.03251953125\\
0.5	0.2376220703125\\
1	0.358251953125\\
1.5	0.404638671875\\
2	0.4264892578125\\
2.5	0.439794921875\\
3	0.452197265625\\
};
\addlegendentry{Empirical error for $p=1024$};
\addplot [color=red,smooth,solid,line width=1.0pt,mark size=2.0pt,mark=triangle,mark options={solid,rotate=270}]
  table[row sep=crcr]{%
-3	0.45168327502439\\
-2.5	0.442082772474596\\
-2	0.427747670868385\\
-1.5	0.404076598017862\\
-1	0.357858640171627\\
-0.5	0.233275225532111\\
0	0\\
0.5	0.233275225532111\\
1	0.357858640171627\\
1.5	0.404076598017862\\
2	0.427747670868385\\
2.5	0.442082772474596\\
3	0.45168327502439\\
};
\addlegendentry{Theoretical error};

\end{axis}
\end{tikzpicture}%
\caption{Performance of LS-SVM, $c_0=1/4,\ c_{1}=c_{2}=1/2,\gamma=1$, polynomial kernel with $\ftau=4,\ \ffftau=2$. $\bx\in\mathcal{N}(\bmu_a,\bC_a)$, with $\bmu_1=\bmu_2=\mathbf{0}_p$, ${\bf C}_{1}={\bf I}_{p}$ and $\{{\bf C}_{2}\}_{i,j}=.4^{\mid i-j\mid}$.}\label{fig:orth}
\end{minipage}
\end{figure}

\footnotetext{Unless particularly stated, the classification error will be understood as $c_1\rm{P}(g({\bx})>\xi_{n}|\bx\in\mathcal{C}_{1})+c_2{\rm P}(g({\bx})<\xi_{n}|{\bx}\in\mathcal{C}_{2})$.}

\begin{remark}[Condition on Kernel Function $f$]\label{rem:kernel-function}
\normalfont From Theorem~\ref{thm:Gaussian Approximation} and Corollary~\ref{cor:Gaussian Approximation normal}, one observes that $|{\rm E}_1-{\rm E}_2|$ is always proportional to the ``informative'' term $\mathfrak{D}$ and should, for fixed ${\rm Var}_a$, be made as large as possible to avoid the overlap of $g(\bx)$ for $\bx$ from different classes. Since ${\rm Var}_a$ does not depend on the signs of $\fftau$ and $\ffftau$, it is easily deduced that, to achieve optimal classification performance, one needs to choose the kernel function $f$ such that $\ftau>0,\fftau<0$ and $\ffftau>0$.
\end{remark}

\medskip

Incidentally, the condition in Remark~\ref{rem:kernel-function} is naturally satisfied for Gaussian kernel $f(x)=\exp\left(-x/(2\sigma^2)\right)$ for any $\sigma$, meaning that, even without specific tuning of the kernel parameter $\sigma$ through cross validation or other techniques, LS-SVM is expected to perform rather well with a Gaussian kernel (as shown in Figure~\ref{fig:gauss_sigma_loop}), which is not always the case for polynomial kernels. This especially entails, for a second-order polynomial kernel given by $f(x)=a_2 x^2+a_1 x+a_0$, that attention should be paid to meeting the aforementioned condition when tuning the kernel parameters $a_2,a_1$ and $a_0$. Figure~\ref{fig:gauss_poly_loop} attests of this remark with Gaussian input data. A rapid increase in classification error rate can be observed both in theory and in practice as soon as the condition $\fftau<0,\ffftau>0$ is no longer satisfied.

\begin{figure}[htb]
\begin{minipage}[b]{1\linewidth}%
\centering{}
\begin{tikzpicture}[font=\footnotesize]
\pgfplotsset{every major grid/.append style={densely dashed}}
/xlabel near ticks
/ylabel near ticks
\begin{axis}[%
width=7.5cm,
height=4.5cm,
scale only axis,
xmode=log,
log basis x={2},
xmin=0.03125,
xmax=256,
xminorticks=true,
xlabel={$\sigma^2$},
ymin=0,
ymax=0.27,
ylabel={Classification error},
ytick = {0,0.05,0.1,0.15,0.2,0.25},
yticklabels={$0$,$0.05$,$0.1$,$0.15$,$0.2$,$0.25$},
grid=major,
axis background/.style={fill=white},
legend style={at={(0.98,0.98)},anchor=north east,legend cell align=left,align=left,draw=white!15!black}
]
\addplot [color=blue,smooth,solid,line width=1.0pt,mark size=2.0pt,mark=square,mark options={solid}]
  table[row sep=crcr]{%
0.03125	0.205729166666667\\
0.0625	0.129947916666667\\
0.125	0.0779947916666667\\
0.25	0.0548177083333333\\
0.5	0.0615885416666667\\
1	0.0861979166666667\\
2	0.119140625\\
4	0.143880208333333\\
8	0.157291666666667\\
16	0.163411458333333\\
32	0.171875\\
64	0.166015625\\
128	0.182682291666667\\
256	0.171614583333333\\
};
\addlegendentry{Empirical error for $n=256$};

\addplot [color=blue,smooth,dashed,line width=1.0pt,mark size=2.0pt,mark=o,mark options={solid}]
  table[row sep=crcr]{%
0.03125	0.133984375\\
0.0625	0.0997395833333333\\
0.125	0.0671223958333333\\
0.25	0.0548177083333333\\
0.5	0.0577473958333333\\
1	0.0835286458333333\\
2	0.1115234375\\
4	0.135872395833333\\
8	0.149544270833333\\
16	0.156770833333333\\
32	0.1630859375\\
64	0.1654296875\\
128	0.1646484375\\
256	0.1638671875\\
};
\addlegendentry{Empirical error for $n=512$};

\addplot [color=red,smooth,solid,line width=1.0pt,mark size=2.0pt,mark=triangle,mark options={solid,rotate=270}]
  table[row sep=crcr]{%
0.03125	0.0857839239487618\\
0.0625	0.0766797224319311\\
0.125	0.064139958174741\\
0.25	0.0539342227778432\\
0.5	0.0583112336182385\\
1	0.0823559359632419\\
2	0.112311869953727\\
4	0.134937465756942\\
8	0.148550890788372\\
16	0.155963756172566\\
32	0.159823691729607\\
64	0.161792136735953\\
128	0.162785981551089\\
256	0.163285309287987\\
};
\addlegendentry{Theoretical error};

\end{axis}
\end{tikzpicture}
\caption{Performance of LS-SVM, $c_0=2, c_{1}=c_{2}=1/2,\gamma=1$, Gaussian kernel. $\bx\in\mathcal{N}(\bmu_a,\bC_a)$, with $\bmu_a=\left[\mathbf{0}_{a-1};2;\mathbf{0}_{p-a}\right]$, ${\bf C}_{1}={\bf I}_{p}$ and $\{{\bf C}_{2}\}_{i,j}=.4^{\mid i-j\mid}(1+4/\sqrt{p})$. \label{fig:gauss_sigma_loop}}
\end{minipage}
\end{figure}

Note also from both Figure~\ref{fig:orth} and Figure~\ref{fig:gauss_sigma_loop} that, when $n,p$ are doubled (from $2048,512$ to $4\,096,1\,024$ in Figure~\ref{fig:orth} and from $256,512$ to $512,1\,024$ in Figure~\ref{fig:gauss_sigma_loop}), the empirical error becomes closer to the theoretical one, which confirms the asymptotic result as $n,p\to\infty$. 

\begin{figure}[htb]
\begin{minipage}[b]{0.45\linewidth}%
\centering{}
\begin{tikzpicture}[font=\footnotesize]
\pgfplotsset{every major grid/.append style={densely dashed}}
/xlabel near ticks
/ylabel near ticks
\begin{axis}[
width=3.5cm,
height=3cm,
scale only axis,
xmin=-5,
xmax=1,
xlabel={$f^{\prime}(\tau)$},
ymin=0,
ymax=0.6,
ylabel={Classification error},
grid=major,
axis background/.style={fill=white},
legend style={at={(0.03,0.97)},anchor=north west,legend cell align=left,align=left,draw=white!15!black, font=\footnotesize}
]
\addplot [color=blue,smooth,solid,line width=1.0pt,mark size=2.0pt,mark=square,mark options={solid}]
  table[row sep=crcr]{%
-5	0.122135416666667\\
-4	0.118359375\\
-3	0.101953125\\
-2	0.0899739583333333\\
-1	0.0618489583333333\\
-0.5	0.055859375\\
-0.4	0.06640625\\
-0.3	0.0641927083333333\\
-0.2	0.071484375\\
-0.1	0.08125\\
0	0.106119791666667\\
0.1	0.127864583333333\\
0.2	0.158854166666667\\
0.3	0.204166666666667\\
0.4	0.249609375\\
0.5	0.28046875\\
0.6	0.330338541666667\\
0.7	0.3828125\\
0.8	0.409895833333333\\
0.9	0.44765625\\
1	0.489713541666667\\
};
\addlegendentry{Empirical error};

\addplot [color=red,smooth,solid,line width=1.0pt,mark size=2.0pt,mark=triangle,mark options={solid,rotate=270}]
  table[row sep=crcr]{%
-5	0.124912390292323\\
-4	0.116474563145825\\
-3	0.104323935597384\\
-2	0.0861329139611201\\
-1	0.0617814858867468\\
-0.5	0.0575203019489714\\
-0.4	0.0599461354596914\\
-0.3	0.0645678498414669\\
-0.2	0.0722110911994839\\
-0.1	0.0839481020158166\\
0	0.101041257488825\\
0.1	0.124683588648435\\
0.2	0.155493671609363\\
0.3	0.192991462328379\\
0.4	0.235486934658724\\
0.5	0.280550413595428\\
0.6	0.325736940898843\\
0.7	0.369121107485369\\
0.8	0.409479123547842\\
0.9	0.446214593543344\\
1	0.479183805076305\\
};
\addlegendentry{Theoretical error};
\end{axis}
\end{tikzpicture}
\centerline{(a) $f(\tau)=4,\ \ffftau=1$}\medskip{}
\end{minipage}\hfill{}%
\begin{minipage}[b]{0.48\linewidth}%
\centering{}
\begin{tikzpicture}[font=\footnotesize]
\pgfplotsset{every major grid/.append style={densely dashed}}
/xlabel near ticks
/ylabel near ticks
\begin{axis}[%
width=3.5cm,
height=3cm,
scale only axis,
xmin=-1,
xmax=5,
xlabel={$f^{\prime\prime}(\tau)$},
ymin=0,
ymax=0.6,
grid=major,
axis background/.style={fill=white},
legend style={at={(0.97,0.97)},anchor=north east,legend cell align=left,align=left,draw=white!15!black, font=\footnotesize}
]
\addplot [color=blue,smooth,solid,line width=1.0pt,mark size=2.0pt,mark=square,mark options={solid}]
  table[row sep=crcr]{%
-1	0.522005208333333\\
-0.9	0.486979166666667\\
-0.8	0.450520833333333\\
-0.7	0.41328125\\
-0.6	0.386197916666667\\
-0.5	0.341536458333333\\
-0.4	0.298177083333333\\
-0.3	0.269401041666667\\
-0.2	0.231119791666667\\
-0.1	0.197916666666667\\
0	0.164583333333333\\
0.1	0.142447916666667\\
0.2	0.125130208333333\\
0.3	0.114453125\\
0.4	0.0932291666666667\\
0.5	0.0854166666666667\\
1	0.06015625\\
2	0.058984375\\
3	0.06484375\\
4	0.06796875\\
5	0.0684895833333333\\
};
\addlegendentry{Empirical error};

\addplot [color=red,smooth,solid,line width=1.0pt,mark size=2.0pt,mark=triangle,mark options={solid,rotate=270}]
  table[row sep=crcr]{%
-1	0.520816194923695\\
-0.9	0.488430148831111\\
-0.8	0.453638024056461\\
-0.7	0.416785778746669\\
-0.6	0.378434652227475\\
-0.5	0.339369250687701\\
-0.4	0.300564915617891\\
-0.3	0.263103739994676\\
-0.2	0.228044575189662\\
-0.1	0.196274843332875\\
0	0.168388301451422\\
0.1	0.144628670317083\\
0.2	0.124912390292323\\
0.3	0.108911047964176\\
0.4	0.0961561582385449\\
0.5	0.0861329139611201\\
1	0.0617814858867468\\
2	0.0575203019489714\\
3	0.0627372684437821\\
4	0.0679505719090806\\
5	0.0722110911994839\\
};
\addlegendentry{Theoretical error};
\end{axis}
\end{tikzpicture}
\centerline{(b) $f(\tau)=4,\ f^{\prime}(\tau)=-1$}\medskip{}
\end{minipage}
\caption{Performance of LS-SVM, $n=256,\ p=512,\ c_{1}=c_{2}=1/2,\gamma=1$, polynomial kernel. $\bx\in\mathcal{N}(\bmu_a,\bC_a)$,  with $\bmu_a=\left[\mathbf{0}_{a-1};2;\mathbf{0}_{p-a}\right]$, ${\bf C}_{1}={\bf I}_{p}$ and $\{{\bf C}_{2}\}_{i,j}=.4^{\mid i-j\mid}(1+4/\sqrt{p})$.}\label{fig:gauss_poly_loop}
\end{figure}

Clearly, for practical use, one needs to know in advance the value of $\tau$ before training so that the kernel $f$ can be properly chosen during the training step. The estimation of $\tau$ is possible, in the large $n,p$ regime, with the following lemma.

\begin{lemma}\label{lem:estim-tau}
Under Assumptions~\ref{as:Growth rate} and~\ref{as:Kernel function}, as $n\to\infty$,
\begin{equation}
\frac{2}{n}\sum_{i=1}^n \frac{\|\bx_i-\bar{\bx}\|^2}{p}\asc\tau
\end{equation}
with $\bar{\bx}\triangleq\frac{1}{n}\sum_{i=1}^n{\bx_i}$.
\end{lemma}

\begin{IEEEproof}
Since
\begin{equation*}
\frac{2}{n}\sum_{i=1}^n \frac{\|\bx_i-\bar{\bx}\|^2}{p}=\frac{2c_1c_2\|\bmu_2-\bmu_1\|^2}{p}+\frac{2}{n}\sum_{i=1}^n\|\bomega_i-\bar{\bomega}\|^2+\kappa
\end{equation*}
with $\kappa=\frac{4}{n\sqrt{p}}(\bmu_2-\bmu_1)^{\T}\left(-c_2\sum_{\bx_i\in\mathcal{C}_1}\bomega_i+c_1\sum_{\bx_j\in\mathcal{C}_2}\bomega_j\right)$ and $\bar{\bomega}=\frac{1}{n}\sum_{i=1}^n\bomega_i$.

According to Assumption~\ref{as:Growth rate} we have $\frac{2c_1c_2}{p}\|\bmu_2-\bmu_1\|^2=O(n^{-1})$. The term $\kappa$ is a linear combination of independent zero-mean Gaussian variables and thus $\kappa\sim\mathcal{N}(0,{\rm Var}[\kappa])$ with ${\rm Var}[\kappa]=\frac{16c_1c_2}{np^2}(\bmu_2-\bmu_1)^{\T}\left(c_2\bC_1+c_1\bC_2\right)(\bmu_2-\bmu_1)=O(n^{-3})$. We thus deduce from Chebyshev's inequality and Borel-Cantelli lemma that $\kappa\asc0$.

We then work on the last term $\frac{2}{n}\sum_{i=1}^n\|\bomega_i-\bar{\bomega}\|^2$ as
\[
\frac{2}{n}\sum_{i=1}^n\|\bomega_i-\bar{\bomega}\|^2=\frac{2}{n}\sum_{i=1}^n\|\bomega_i\|^2-2\|\bar{\bomega}\|^2.
\]
Since $\bar{\bomega}\sim\mathcal{N}(\mathbf{0},\bC^\circ/np)$, we deduce that $\|\bar{\bomega}\|^2\asc0$. Ultimately by the strong law of large numbers, the term $\frac{2}{n}\sum_{i=1}^n\|\bomega_i\|^2\asc\tau$, which concludes the proof.
\end{IEEEproof}

\subsection{Some limiting cases}
\label{subsec:limiting-cases}

\subsubsection{Dominant information in means}
When $\|\bmu_2-\bmu_1\|^2$ is largely dominant over $(\tr({\bf C}_{2}-{\bf C}_{1}))^{2}/p$ and $\tr((\mathbf{C}_{2}-\mathbf{C}_{1})^{2})/p$, from Theorem~\ref{thm:Gaussian Approximation}, both $\rm{E}_a-(c_2-c_1)$ and $\sqrt{\rm{Var}_a}$ are (approximately) proportional to $\fftau$, which eventually makes the choice of the kernel irrelevant (as long as $\fftau\neq 0$). This result also holds true for $\rm{E}_a^*$ and $\sqrt{\rm{Var}_a^*}$ when normalized labels $\by^*$ are applied, as a result of Lemma~\ref{lem:norml-y}.
\subsubsection{\texorpdfstring{$c_0$}{c0} large or small}
Note that, different from both $\mathcal{V}_1$ and $\mathcal{V}_2$, $\mathcal{V}_3$ is a function of $c_0$ as it can be rewritten as 
\begin{equation*}
\mathcal{V}_{3}^{a} =\frac{2c_0\left(\fftau\right)^{2}}{p^{3}}\left(\frac{\tr\bC_{1}\bC_{a}}{c_{1}}+\frac{\tr\bC_{2}\bC_{a}}{c_{2}}\right)
\end{equation*}
which indicates that the variance of $g(\bx)$ grows as $c_0$ becomes large. This result is easily understood since, with $p$ fixed, a small $c_0$ means a larger $n$, and with more training samples, one may ``gain'' more information of the two different classes, which reduces the ``uncertainty'' of the classifier. When $n\to\infty$ with a fixed $p$, we have $c_0 \to 0$ and the LS-SVM is considered ``well-trained'' and its performance can be described with Theorem~\ref{thm:Gaussian Approximation} by taking $\mathcal{V}_3=0$. However, it is worthy noting that the misclassification rate may not be $0$ even in this case, since $\mathcal{V}_1$ and $\mathcal{V}_2$ may differ from $0$, which indicated the theoretical limitation of LS-SVM in separating high dimensional Gaussian vectors. On the contrary, when $c_0\to\infty$, with few training data, LS-SVM does not sample sufficiently the high dimensional space of the $\bx$'s, thus resulting in a classifier with arbitrarily large variance (for fixed means). Moreover, since the term $\mathcal{V}_3^a$ is proportional to $n^{-1}$, we see that for $\fftau$ away from zero and fixed large $p$, as $n$ grows large, the two Gaussians $G_1$ and $G_2$ in Theorem~\ref{thm:Gaussian Approximation} separate from each other at a rate of $n^{-\frac12}$, the overlapping section of the Gaussian tails then provides the misclassification rate via Corollary~\ref{cor:Asymptotic Error Rate}. Figure~\ref{fig:gauss-loop-c0} confirms this result with $p$ fixed to $256$ while $n$ varies from $8$ to $8\,192$.

\begin{figure}[htb]
\noindent\begin{minipage}[b]{1\linewidth}%
\centering{}
\begin{tikzpicture}[font=\footnotesize]
\pgfplotsset{every major grid/.append style={densely dashed}}
/xlabel near ticks
/ylabel near ticks
\begin{axis}[%
width=7.5cm,
height=4.5cm,
scale only axis,
xmode=log,
log basis x={2},
xmin=0.03125,
xmax=32,
xlabel={$c_0=\frac{p}{n}$},
ymin=0,
ymax=0.35,
ytick = {0,0.05,0.1,0.15,0.20,0.25,0.30,0.35},
yticklabels={$0$,$0.05$,$0.1$,$0.15$,$0.20$,$0.25$,$0.30$,$0.35$},
ylabel={Classification error},
axis background/.style={fill=white},
grid=major,
legend style={at={(0.02,0.75)},anchor=south west,legend cell align=left,align=left,draw=white!15!black}
]
\addplot [color=blue,smooth,solid,line width=1.0pt,mark size=2.0pt,mark=square,mark options={solid}]
  table[row sep=crcr]{%
0.03125	0.0245068791730966\\
0.0625	0.0254791516344504\\
0.125	0.0274429523823023\\
0.25	0.0314310235624709\\
0.5	0.0395396680343861\\
1	0.0556944130595506\\
2	0.0906\\
4    	0.1425\\
8    	0.2151\\
16    	0.2781\\
32    	0.3417\\
};
\addlegendentry{Empirical error};
\addplot [color=red,smooth,solid,line width=1.0pt,mark size=2.0pt,mark=triangle,mark options={solid,rotate=270}]
  table[row sep=crcr]{%
0.03125	0.0281534830729167\\
0.0625	0.0292643229166667\\
0.125	0.0312337239583333\\
0.25	0.0350260416666667\\
0.5	0.044140625\\
1	0.0640625\\
2	0.0856\\
4    	0.1334\\
8        	0.1957\\
16        0.2618\\
32         0.3214\\
};
\addlegendentry{Theoretical error};

\end{axis}
\end{tikzpicture}%
\caption{Performance of LS-SVM, $p=256,\ c_{1}=c_{2}=1/2,\gamma=1$, Gaussian kernel with $\sigma^2=1$. $\bx\in\mathcal{N}(\bmu_a,\bC_a)$, with $\bmu_a=\left[\mathbf{0}_{a-1};2;\mathbf{0}_{p-a}\right]$, ${\bf C}_{1}={\bf I}_{p}$ and $\{{\bf C}_{2}\}_{i,j}=.4^{\mid i-j\mid}(1+4/\sqrt{p})$. \label{fig:gauss-loop-c0}}
\end{minipage}
\end{figure}

\subsubsection{\texorpdfstring{$c_1\to0$}{c1to0} }
As revealed in Remark~\ref{rem:Dominant Bias}, the ratio $c_1/c_2$ plays a significant role in the performance of classification. A natural question arises: what happens when one class is strongly dominant over the other? Take the case of $c_1\to 0, c_2\to 1$. From Corollary~\ref{cor:Gaussian Approximation normal}, one has $\rm{E}_1^*\to -\gamma\mathfrak{D}$, $\rm{E}_2^*\to 0$ and $\mathcal{V}_3^a\to\infty$ because of $c_1\to 0$ in the denominator, which then makes the ratio $\frac{\rm{E}_a^*}{\sqrt{\rm{Var}_a^*}}$ (and thus  $\frac{\rm{E}_a-(c_2-c_1)}{\sqrt{\rm{Var}_a}}$) go to zero, resulting in a poorly-performing LS-SVM. The same occurs when $c_1\to 1$ and $ c_2\to 0$. Figure~\ref{fig:gauss-loop-c1} collaborates this remark with $c_1=1/32$ in subfigure (a) and $1/2$ in (b). Note that in subfigure (a), even with a smartly chosen threshold $\xi$, LS-SVM is impossible to perform as well as in the case $c_1=c_2$, as a result of the significant overlap between the two histograms.

\begin{figure}[htb]
\begin{minipage}[b]{0.48\linewidth}%
\centering{}
\definecolor{mycolor1}{rgb}{0.00000,1.00000,1.00000}%
\definecolor{mycolor2}{rgb}{1.00000,0.00000,1.00000}%
\begin{tikzpicture}[font=\footnotesize]
\begin{axis}[%
width=4.5cm,
height=3cm,
scale only axis,
xmin=0.9355,
xmax=0.939,
xtick = {0.936,0.937,0.938},
xticklabels={$0.936$,$0.937$,$0.938$},
every outer y axis line/.append style={white},
every y tick label/.append style={font=\color{white}},
ymin=0,
ymax=1200,
ytick={\empty},
axis background/.style={fill=white},
axis x line*=bottom,
axis y line*=left,
]
\addplot[ybar,bar width=2,draw=white,fill=blue,area legend] plot table[row sep=crcr] {%
0.935575154991359	28.0202674157689\\
0.935664376124152	0\\
0.935753597256945	0\\
0.935842818389738	0\\
0.935932039522531	28.0202674157689\\
0.936021260655324	112.081069663076\\
0.936110481788117	0\\
0.93619970292091	56.0405348315378\\
0.936288924053704	252.18240674192\\
0.936378145186497	420.304011236533\\
0.93646736631929	364.263476404996\\
0.936556587452083	532.385080899609\\
0.936645808584876	728.526952809991\\
0.936735029717669	868.628289888836\\
0.936824250850462	840.608022473067\\
0.936913471983255	1008.72962696768\\
0.937002693116048	1176.85123146229\\
0.937091914248841	1008.72962696768\\
0.937181135381634	644.466150562685\\
0.937270356514427	868.628289888836\\
0.93735957764722	588.425615731147\\
0.937448798780013	560.405348315378\\
0.937538019912806	532.385080899609\\
0.937627241045599	168.121604494613\\
0.937716462178392	168.121604494613\\
0.937805683311185	112.081069663076\\
0.937894904443978	84.0608022473067\\
0.937984125576771	0\\
0.938073346709564	28.0202674157689\\
0.938162567842358	28.0202674157689\\
};
\addplot[ybar,bar width=2,draw=white,fill=red,area legend] plot table[row sep=crcr] {%
0.935795544961955	0.745263878128964\\
0.935903755172669	0\\
0.936011965383384	0.745263878128964\\
0.936120175594098	3.72631939064482\\
0.936228385804813	6.70737490316068\\
0.936336596015528	16.3958053188372\\
0.936444806226242	41.734777175222\\
0.936553016436957	63.347429640962\\
0.936661226647671	119.242220500634\\
0.936769436858386	216.871788535529\\
0.936877647069101	327.170842498615\\
0.936985857279815	479.949937515053\\
0.93709406749053	664.775379291036\\
0.937202277701244	847.365029432632\\
0.937310487911959	1010.57781874288\\
0.937418698122674	1079.14209553074\\
0.937526908333388	1050.82206816184\\
0.937635118544103	959.899875030106\\
0.937743328754817	735.575447713288\\
0.937851538965532	612.606907822009\\
0.937959749176246	442.686743608605\\
0.938067959386961	266.05920449204\\
0.938176169597676	149.798039503922\\
0.93828437980839	83.469554350444\\
0.938392590019105	37.2631939064482\\
0.938500800229819	14.1600136844503\\
0.938609010440534	6.70737490316068\\
0.938717220651249	2.98105551251586\\
0.938825430861963	0\\
0.938933641072678	0.745263878128964\\
};
\addplot [color=mycolor1,dashed,line width=2.0pt,forget plot]
  table[row sep=crcr]{%
0.935575154991359	0.0241108407600384\\
0.935664376124152	0.0856808239528971\\
0.935753597256945	0.281436289449637\\
0.935842818389738	0.854479682729518\\
0.935932039522531	2.39799662685222\\
0.936021260655324	6.22043243800178\\
0.936110481788117	14.9148115211948\\
0.93619970292091	33.0552270176687\\
0.936288924053704	67.7154394626382\\
0.936378145186497	128.221378303842\\
0.93646736631929	224.418349073708\\
0.936556587452083	363.062561914708\\
0.936645808584876	542.912392694759\\
0.936735029717669	750.417969094186\\
0.936824250850462	958.742273842001\\
0.936913471983255	1132.20667755376\\
0.937002693116048	1235.87551821697\\
0.937091914248841	1246.9496828878\\
0.937181135381634	1162.91590807034\\
0.937270356514427	1002.4734430476\\
0.93735957764722	798.771615566417\\
0.937448798780013	588.298246209735\\
0.937538019912806	400.495523493806\\
0.937627241045599	252.013009775686\\
0.937716462178392	146.579568412681\\
0.937805683311185	78.8041517510907\\
0.937894904443978	39.1606544586813\\
0.937984125576771	17.9877136879934\\
0.938073346709564	7.63707707879133\\
0.938162567842358	2.99711608135245\\
};
\addplot [color=mycolor2,dashed,line width=2.0pt,forget plot]
  table[row sep=crcr]{%
0.935795544961955	0.00401840454060586\\
0.935903755172669	0.0186279029197551\\
0.936011965383384	0.0781579781074563\\
0.936120175594098	0.29681217132132\\
0.936228385804813	1.02020896621651\\
0.936336596015528	3.17391729692232\\
0.936444806226242	8.93719257780433\\
0.936553016436957	22.7774795943791\\
0.936661226647671	52.5423213568464\\
0.936769436858386	109.701328350401\\
0.936877647069101	207.306804094124\\
0.936985857279815	354.580025209509\\
0.93709406749053	548.926271656499\\
0.937202277701244	769.153121217994\\
0.937310487911959	975.462631421154\\
0.937418698122674	1119.71502780451\\
0.937526908333388	1163.33131047985\\
0.937635118544103	1093.95225334272\\
0.937743328754817	931.091490028836\\
0.937851538965532	717.274493455943\\
0.937959749176246	500.123725415772\\
0.938067959386961	315.622944928665\\
0.938176169597676	180.284636198061\\
0.93828437980839	93.2068665493347\\
0.938392590019105	43.6150262463702\\
0.938500800229819	18.4724035782547\\
0.938609010440534	7.0812465967088\\
0.938717220651249	2.45694292927039\\
0.938825430861963	0.771577329708999\\
0.938933641072678	0.219312249248852\\
};
\end{axis}
\end{tikzpicture}
\centerline{(a) $c_1=1/32$}\medskip{}
\end{minipage}\hfill{}%
\begin{minipage}[b]{0.48\linewidth}%
\centering{}
\definecolor{mycolor1}{rgb}{0.00000,1.00000,1.00000}%
\definecolor{mycolor2}{rgb}{1.00000,0.00000,1.00000}%
\begin{tikzpicture}[font=\footnotesize]
\begin{axis}[%
width=4.5cm,
height=3cm,
scale only axis,
scaled x ticks = false,
xmin=-0.008,
xmax=0.008,
x tick label style={/pgf/number format/fixed},
every outer y axis line/.append style={white},
every y tick label/.append style={font=\color{white}},
ymin=0,
ymax=300,
ytick={\empty},
axis background/.style={fill=white},
axis x line*=bottom,
axis y line*=left,
]
\addplot[ybar,bar width=2,draw=white,fill=blue,area legend] plot table[row sep=crcr] {%
-0.00747959251665285	0.812813288529122\\
-0.00709512537713856	0.406406644264561\\
-0.00671065823762426	2.03203322132281\\
-0.00632619109810996	4.87687973117473\\
-0.00594172395859567	10.9729793951431\\
-0.00555725681908137	25.1972119444028\\
-0.00517278967956707	39.827851137927\\
-0.00478832254005278	62.1802165724779\\
-0.00440385540053848	93.0671215365845\\
-0.00401938826102418	132.082159385982\\
-0.00363492112150989	186.947056361698\\
-0.00325045398199559	215.801928104482\\
-0.00286598684248129	259.28743904079\\
-0.002481519702967	258.881032396525\\
-0.0020970525634527	274.730891522843\\
-0.0017125854239384	255.629779242409\\
-0.00132811828442411	233.683820452123\\
-0.000943651144909809	175.974076966555\\
-0.000559184005395513	134.114192607305\\
-0.000174716865881216	92.6607148923199\\
0.00020975027363308	54.8648969757158\\
0.000594217413147377	37.3894112723396\\
0.000978684552661674	22.3523654345509\\
0.00136315169217597	15.0370458377888\\
0.00174761883169027	7.3153195967621\\
0.00213208597120456	2.84484650985193\\
0.00251655311071886	1.21921993279368\\
0.00290102025023316	0\\
0.00328548738974745	0.406406644264561\\
0.00366995452926175	0.406406644264561\\
};
\addplot[ybar,bar width=2,draw=white,fill=red,area legend] plot table[row sep=crcr] {%
-0.00327095501014032	0.419540949519955\\
-0.00289852414211949	2.93678664663969\\
-0.00252609327409866	4.61495044471951\\
-0.00215366240607783	10.9080646875188\\
-0.001781231538057	13.0057694351186\\
-0.00140880067003617	30.2069483654368\\
-0.00103636980201534	41.9540949519955\\
-0.000663938933994514	71.7415023679123\\
-0.000291508065973685	99.0116640867094\\
8.09228020471455e-05	146.839332331984\\
0.000453353670067975	192.149754880139\\
0.000825784538088804	230.327981286455\\
0.00119821540610963	247.109619267254\\
0.00157064627413046	288.224632320209\\
0.00194307714215129	282.35105902693\\
0.00231550801017212	240.396964074934\\
0.00268793887819295	217.322211851337\\
0.00306036974621378	193.408377728699\\
0.00343280061423461	128.799071502626\\
0.00380523148225544	91.4599269953502\\
0.00417766235027627	68.3851747717527\\
0.0045500932182971	36.0805216587162\\
0.00492252408631793	23.0747522235975\\
0.00529495495433876	13.8448513341585\\
0.00566738582235959	5.87357329327937\\
0.00603981669038042	2.51724569711973\\
0.00641224755840125	0.419540949519955\\
0.00678467842642208	1.25862284855987\\
0.00715710929444291	0\\
0.00752954016246374	0.419540949519955\\
};
\addplot [color=mycolor1,dashed,line width=2.0pt,forget plot]
  table[row sep=crcr]{%
-0.00747959251665285	0.150433141356924\\
-0.00709512537713856	0.417690686852108\\
-0.00671065823762426	1.07671972737247\\
-0.00632619109810996	2.57683825334012\\
-0.00594172395859567	5.72543196747701\\
-0.00555725681908137	11.8104370392231\\
-0.00517278967956707	22.6183171662351\\
-0.00478832254005278	40.2152915265799\\
-0.00440385540053848	66.3832886411952\\
-0.00401938826102418	101.733249495208\\
-0.00363492112150989	144.745036587861\\
-0.00325045398199559	191.196991704679\\
-0.00286598684248129	234.474187661327\\
-0.002481519702967	266.959657048993\\
-0.0020970525634527	282.18429685734\\
-0.0017125854239384	276.921489835923\\
-0.00132811828442411	252.299905210893\\
-0.000943651144909809	213.409685581713\\
-0.000559184005395513	167.58987431361\\
-0.000174716865881216	122.185039976237\\
0.00020975027363308	82.7036936183554\\
0.000594217413147377	51.9718752524172\\
0.000978684552661674	30.3213462240816\\
0.00136315169217597	16.4234798590201\\
0.00174761883169027	8.25882941121021\\
0.00213208597120456	3.85574602823804\\
0.00251655311071886	1.67122511971159\\
0.00290102025023316	0.672508999385299\\
0.00328548738974745	0.251245248164094\\
0.00366995452926175	0.0871433444547605\\
};
\addplot [color=mycolor2,dashed,line width=2.0pt,forget plot]
  table[row sep=crcr]{%
-0.00327095501014032	0.655562240040019\\
-0.00289852414211949	1.48287016244432\\
-0.00252609327409866	3.15965241952699\\
-0.00215366240607783	6.34194514908504\\
-0.001781231538057	11.9909237238034\\
-0.00140880067003617	21.3564855796773\\
-0.00103636980201534	35.8305868092644\\
-0.000663938933994514	56.6271940621905\\
-0.000291508065973685	84.3030236665521\\
8.09228020471455e-05	118.224707543789\\
0.000453353670067975	156.178171075962\\
0.000825784538088804	194.347699689115\\
0.00119821540610963	227.816651489571\\
0.00157064627413046	251.558218995654\\
0.00194307714215129	261.660728391062\\
0.00231550801017212	256.380841292683\\
0.00268793887819295	236.635325776875\\
0.00306036974621378	205.740871741738\\
0.00343280061423461	168.503358238187\\
0.00380523148225544	130.000054187495\\
0.00417766235027627	94.4768693043375\\
0.0045500932182971	64.6776812022719\\
0.00492252408631793	41.7090575539975\\
0.00529495495433876	25.3368896385367\\
0.00566738582235959	14.4985045876047\\
0.00603981669038042	7.81520021370727\\
0.00641224755840125	3.96829540582133\\
0.00678467842642208	1.89808154927532\\
0.00715710929444291	0.855209917309999\\
0.00752954016246374	0.362975721316017\\
};
\end{axis}
\end{tikzpicture}
\centerline{(b) $c_1=1/2$}\medskip{}
\end{minipage}
\caption{Gaussian approximation of $g({\bf x})$, $n=256,\ p=512,\ c_2=1-c_1,\gamma=1$, Gaussian kernel with $\sigma^2=1$. $\bx\in\mathcal{N}(\bmu_a,\bC_a)$, with $\bmu_a=\left[\mathbf{0}_{a-1};2;\mathbf{0}_{p-a}\right]$, ${\bf C}_{1}={\bf I}_{p}$ and $\{{\bf C}_{2}\}_{i,j}=.4^{\mid i-j\mid}(1+4/\sqrt{p})$.\label{fig:gauss-loop-c1}}
\end{figure}

\subsection{Applying to real-world datasets}
When the classification performance of real-world datasets is concerned, our theory may be limited by: i) the fact that it is an asymptotic result and allows for an estimation error of order $O(n^{-\frac12})$ between theory and practice and ii) the strong Gaussian assumption for the input data. 

However, when applied to real-world datasets, here to the popular MNIST \cite{lecun1998gradient} and Fashion-MNIST \cite{xiao2017fashion} datasets, our asymptotic results, which are theoretically only applicable for Gaussian data, show an unexpectedly similar behavior. Here we consider a two-class classification problem with a training set of $n=256$ vectorized images of size $p=784$ randomly selected from the MNIST and Fashion-MNIST datasets (numbers $8$ and $9$ in both cases as an example). Then a test set of $n_{\rm test}=256$ is used to evaluate the classification performance. Means and covariances are empirically obtained from the full set of $11\,800$ MNIST images ($5\,851$ images of number $8$ and $5\,949$ of number $9$) and of $11\,800$ Fashion-MNIST images ($5\,851$ images of number $8$ and $5\,949$ of number $9$), respectively. Despite the obvious non-Gaussianity as well as the clearly different nature of the input data (from the two datasets), the distribution of $g(\bx)$ is still surprisingly close to its Gaussian approximation computed from Theorem~\ref{thm:Gaussian Approximation}, as shown in Figure~\ref{fig:Gaussian-approximation-MINIST} and~\ref{fig:Gaussian-approximation-Fashion-MINIST} for MNIST and Fashion-MNIST, respectively. In both cases we plot the results from (a) raw images as well as (b) when Gaussian white noise is artificially added to the image vectors.

\begin{figure}[htb]
\begin{minipage}[b]{0.48\linewidth}%
\centering{}
\definecolor{mycolor1}{rgb}{0.00000,1.00000,1.00000}%
\definecolor{mycolor2}{rgb}{1.00000,0.00000,1.00000}%
\begin{tikzpicture}[font=\footnotesize]
\begin{axis}[%
width=4.5cm,
height=3cm,
scale only axis,
scaled x ticks = false,
xmin=-0.08,
xmax=0.08,
x tick label style={/pgf/number format/fixed},
every outer y axis line/.append style={white},
every y tick label/.append style={font=\color{white}},
ymin=0,
ymax=30,
ytick={\empty},
axis background/.style={fill=white},
axis x line*=bottom,
axis y line*=left,
]
\addplot[ybar,bar width=2,draw=white,fill=blue,area legend] plot table[row sep=crcr] {%
-0.0761123310464332	0.756253825488907\\
-0.0723244714194338	0.825004173260626\\
-0.0685366117924343	2.95626495418391\\
-0.0647487521654348	4.53752295293344\\
-0.0609608925384353	7.56253825488907\\
-0.0571730329114359	11.2063066867902\\
-0.0533851732844364	14.5063233798327\\
-0.0495973136574369	19.0438463327661\\
-0.0458094540304374	19.8688505060267\\
-0.0420215944034379	20.9001057226025\\
-0.0382337347764385	18.5625938983641\\
-0.034445875149439	20.2126022448853\\
-0.0306580155224395	19.5250987671681\\
-0.02687015589544	20.0063512015702\\
-0.0230822962684406	17.8063400728752\\
-0.0192944366414411	12.7875646855397\\
-0.0155065770144416	13.4750681632569\\
-0.0117187173874421	9.69379903581235\\
-0.00793085776044265	8.38754242814969\\
-0.00414299813344317	6.05003060391125\\
-0.000355138506443697	5.01877538733547\\
0.00343272112055579	3.16251599749906\\
0.00722058074755524	2.75001391086875\\
0.0110084403745547	1.37500695543438\\
0.0147963000015542	1.16875591211922\\
0.0185841596285537	0.756253825488907\\
0.0223720192555532	0.55000278217375\\
0.0261598788825526	0.412502086630313\\
0.0299477385095521	0.0687503477717188\\
0.0337355981365516	0.0687503477717188\\
};
\addplot[ybar,bar width=2,draw=white,fill=red,area legend] plot table[row sep=crcr] {%
-0.0396403779881778	0.201748202406226\\
-0.0357679766944754	0.403496404812452\\
-0.031895575400773	0.672494008020753\\
-0.0280231741070706	0.470745805614527\\
-0.0241507728133682	0.739743408822828\\
-0.0202783715196659	1.27773861523943\\
-0.0164059702259635	1.54673621844773\\
-0.0125335689322611	0.941491611229054\\
-0.00866116763855869	2.82447483368716\\
-0.0047887663448563	3.16072183769754\\
-0.00091636505115391	4.9092062585515\\
0.00295603624254848	4.77470745694735\\
0.00682843753625087	7.66643169143658\\
0.0107008388299533	8.60792330266564\\
0.0145732401236556	12.1721415451756\\
0.018445641417358	15.198364581269\\
0.0223180427110604	18.0900888157583\\
0.0261904440047628	20.7800648478413\\
0.0300628452984652	22.7975468719035\\
0.0339352465921676	24.142534887945\\
0.03780764788587	24.9495276975699\\
0.0416800491795724	25.891019308799\\
0.0455524504732747	20.4438178438309\\
0.0494248517669771	16.946849002123\\
0.0532972530606795	11.2978993347486\\
0.0571696543543819	4.10221344892659\\
0.0610420556480843	2.55547723047886\\
0.0649144569417867	0.537995206416602\\
0.0687868582354891	0.0672494008020753\\
0.0726592595291915	0.0672494008020753\\
};
\addplot [color=mycolor1,dashed,line width=2.0pt,forget plot]
  table[row sep=crcr]{%
-0.0761123310464332	1.98190587147365\\
-0.0723244714194338	2.97013585455531\\
-0.0685366117924343	4.28375570406921\\
-0.0647487521654348	5.9460446765302\\
-0.0609608925384353	7.94303966643713\\
-0.0571730329114359	10.211754978068\\
-0.0533851732844364	12.6348220944962\\
-0.0495973136574369	15.0450268683738\\
-0.0458094540304374	17.2413748941778\\
-0.0420215944034379	19.0154196283977\\
-0.0382337347764385	20.1834323526335\\
-0.034445875149439	20.6176529051237\\
-0.0306580155224395	20.2692890411629\\
-0.02687015589544	19.1775400573742\\
-0.0230822962684406	17.4623372854159\\
-0.0192944366414411	15.3026602884448\\
-0.0155065770144416	12.9058494053135\\
-0.0117187173874421	10.4751762784461\\
-0.00793085776044265	8.18259729727627\\
-0.00414299813344317	6.1514302949012\\
-0.000355138506443697	4.45057503269785\\
0.00343272112055579	3.09892623300124\\
0.00722058074755524	2.07664107825106\\
0.0110084403745547	1.33926566171816\\
0.0147963000015542	0.831241309998387\\
0.0185841596285537	0.496526684491364\\
0.0223720192555532	0.285438900510571\\
0.0261598788825526	0.157920611592617\\
0.0299477385095521	0.0840852008186828\\
0.0337355981365516	0.0430879076509504\\
};
\addplot [color=mycolor2,dashed,line width=2.0pt,forget plot]
  table[row sep=crcr]{%
-0.0396403779881778	0.00055418435342997\\
-0.0357679766944754	0.00165499721359977\\
-0.031895575400773	0.00465976785380335\\
-0.0280231741070706	0.0123695899248751\\
-0.0241507728133682	0.0309578171295307\\
-0.0202783715196659	0.0730481742798073\\
-0.0164059702259635	0.1625071255321\\
-0.0125335689322611	0.340846986072431\\
-0.00866116763855869	0.674016479274295\\
-0.0047887663448563	1.25662466207917\\
-0.00091636505115391	2.20884199075483\\
0.00295603624254848	3.66056167022305\\
0.00682843753625087	5.71945718308636\\
0.0107008388299533	8.42531031542184\\
0.0145732401236556	11.7014856657595\\
0.018445641417358	15.3221641783218\\
0.0223180427110604	18.9157351782489\\
0.0261904440047628	22.016604758357\\
0.0300628452984652	24.1602527874524\\
0.0339352465921676	24.9963506889537\\
0.03780764788587	24.3823606704137\\
0.0416800491795724	22.4232676318606\\
0.0455524504732747	19.4422291206288\\
0.0494248517669771	15.8934148815489\\
0.0532972530606795	12.249331584609\\
0.0571696543543819	8.900851772743\\
0.0610420556480843	6.09782215090212\\
0.0649144569417867	3.93860062563223\\
0.0687868582354891	2.39846367072907\\
0.0726592595291915	1.37704569182447\\
};
\end{axis}
\end{tikzpicture}
\centerline{(a) without noise}\medskip{}
\end{minipage}\hfill{}%
\begin{minipage}[b]{0.48\linewidth}%
\centering{}
\definecolor{mycolor1}{rgb}{0.00000,1.00000,1.00000}%
\definecolor{mycolor2}{rgb}{1.00000,0.00000,1.00000}%
\begin{tikzpicture}[font=\footnotesize]
\begin{axis}[%
width=4.5cm,
height=3cm,
scale only axis,
scaled x ticks = false,
xmin=-0.03,
xmax=0.03,
x tick label style={/pgf/number format/fixed},
every outer y axis line/.append style={white},
every y tick label/.append style={font=\color{white}},
ymin=0,
ymax=80,
ytick={\empty},
axis background/.style={fill=white},
axis x line*=bottom,
axis y line*=left,
]
\addplot[ybar,bar width=2,draw=white,fill=blue,area legend] plot table[row sep=crcr] {%
-0.0324618289180731	0.194111476730229\\
-0.0308519292600827	0.388222953460458\\
-0.0292420296020924	2.52344919749297\\
-0.027632129944102	5.24100987171618\\
-0.0260222302861116	7.76445906920915\\
-0.0244123306281213	13.0054689409253\\
-0.0228024309701309	21.7404853937856\\
-0.0211925313121405	27.9520526491529\\
-0.0195826316541502	35.1341772881714\\
-0.0179727319961598	46.0044199850642\\
-0.0163628323381694	48.333757705827\\
-0.0147529326801791	55.7099938215757\\
-0.0131430330221887	48.5278691825572\\
-0.0115331333641983	57.2628856354175\\
-0.00992323370620798	49.1102036127479\\
-0.00831333404821762	48.333757705827\\
-0.00670343439022725	34.745954334711\\
-0.00509353473223689	31.4460592302971\\
-0.00348363507424652	21.9345968705159\\
-0.00187373541625616	20.9640394868647\\
-0.000263835758265794	15.5289181384183\\
0.00134606389972457	9.12323940632075\\
0.00295596355771493	7.18212463901846\\
0.00456586321570529	5.82334430190686\\
0.00617576287369566	2.91167215095343\\
0.00778566253168603	1.94111476730229\\
0.00939556218967639	1.55289181384183\\
0.0110054618476668	0.582334430190686\\
0.0126153615056571	0\\
0.0142252611636475	0.194111476730229\\
};
\addplot[ybar,bar width=2,draw=white,fill=red,area legend] plot table[row sep=crcr] {%
-0.0176601116334843	0.579576407496546\\
-0.0160425509669804	0.386384271664364\\
-0.0144249903004766	0.772768543328729\\
-0.0128074296339728	0.386384271664364\\
-0.011189868967469	1.15915281499309\\
-0.00957230830096517	1.35234495082527\\
-0.00795474763446136	1.73872922248964\\
-0.00633718696795754	4.05703485247582\\
-0.00471962630145373	6.76172475412637\\
-0.00310206563494991	8.30726184078383\\
-0.0014845049684461	11.3983360140987\\
0.000133055698057719	13.9098337799171\\
0.00175061636456153	20.6715585340435\\
0.00336817703106535	27.626475424002\\
0.00498573769756916	35.1609687214571\\
0.00660329836407298	43.468230562241\\
0.0082208590305768	54.6733744405075\\
0.00983841969708061	59.8895621079765\\
0.0114559803635844	70.5151295787465\\
0.0130735410300882	61.6282913304661\\
0.0146911016965921	65.2989419112776\\
0.0163086623630959	49.4571867730386\\
0.0179262230295997	30.3311653256526\\
0.0195437836961035	23.9558248431906\\
0.0211613443626073	14.875794459078\\
0.0227789050291111	7.14810902579074\\
0.024396465695615	1.54553708665746\\
0.0260140263621188	0.965960679160911\\
0.0276315870286226	0\\
0.0292491476951264	0.193192135832182\\
};
\addplot [color=mycolor1,dashed,line width=2.0pt,forget plot]
  table[row sep=crcr]{%
-0.0324618289180731	1.39900936709009\\
-0.0308519292600827	2.4095933198976\\
-0.0292420296020924	3.97055664352574\\
-0.027632129944102	6.2595564267803\\
-0.0260222302861116	9.44104871875889\\
-0.0244123306281213	13.6232722585133\\
-0.0228024309701309	18.8073302435529\\
-0.0211925313121405	24.8403312155557\\
-0.0195826316541502	31.3886134246441\\
-0.0179727319961598	37.9464714153236\\
-0.0163628323381694	43.8889490319813\\
-0.0147529326801791	48.5650099649571\\
-0.0131430330221887	51.4133970797094\\
-0.0115331333641983	52.0731243370503\\
-0.00992323370620798	50.4586338961864\\
-0.00831333404821762	46.778022483852\\
-0.00670343439022725	41.4889788139174\\
-0.00509353473223689	35.2053081432183\\
-0.00348363507424652	28.5803847951085\\
-0.00187373541625616	22.1979316927611\\
-0.000263835758265794	16.4945877969304\\
0.00134606389972457	11.7261363671097\\
0.00295596355771493	7.97540839679922\\
0.00456586321570529	5.18961857374122\\
0.00617576287369566	3.23074339109518\\
0.00778566253168603	1.92421682259632\\
0.00939556218967639	1.09645311831685\\
0.0110054618476668	0.597737694493644\\
0.0126153615056571	0.311756658942485\\
0.0142252611636475	0.155562656254379\\
};
\addplot [color=mycolor2,dashed,line width=2.0pt,forget plot]
  table[row sep=crcr]{%
-0.0176601116334843	0.00199085514544633\\
-0.0160425509669804	0.00598848439012986\\
-0.0144249903004766	0.0169301932971602\\
-0.0128074296339728	0.0449857169385218\\
-0.011189868967469	0.112345352362296\\
-0.00957230830096517	0.263695859135967\\
-0.00795474763446136	0.581727030572258\\
-0.00633718696795754	1.20615442167644\\
-0.00471962630145373	2.35046786045651\\
-0.00310206563494991	4.30500297497869\\
-0.0014845049684461	7.41071871899399\\
0.000133055698057719	11.9898828010309\\
0.00175061636456153	18.2321246456947\\
0.00336817703106535	26.0571767844127\\
0.00498573769756916	35.0013803899797\\
0.00660329836407298	44.1886486394352\\
0.0082208590305768	52.4329150170904\\
0.00983841969708061	58.4742985326346\\
0.0114559803635844	61.2905859171347\\
0.0130735410300882	60.3796039993699\\
0.0146911016965921	55.90549311333\\
0.0163086623630959	48.6504027057873\\
0.0179262230295997	39.7911170648083\\
0.0195437836961035	30.5881744276692\\
0.0211613443626073	22.0998191450049\\
0.0227789050291111	15.0069216816827\\
0.024396465695615	9.57772244370471\\
0.0260140263621188	5.74513996308385\\
0.0276315870286226	3.23896842566112\\
0.0292491476951264	1.71624999998823\\
};
\end{axis}
\end{tikzpicture}
\centerline{(b) with $0\db$ noise}\medskip{}
\end{minipage}
\caption{Gaussian approximation of $g({\bx})$, $n=256,\ p=784,\ c_{1}=c_{2}=\frac{1}{2},\gamma=1$,
Gaussian kernel with $\sigma=1$, MNIST data (numbers $8$ and $9$) without and with $0\db$ noise.\label{fig:Gaussian-approximation-MINIST} }
\end{figure}
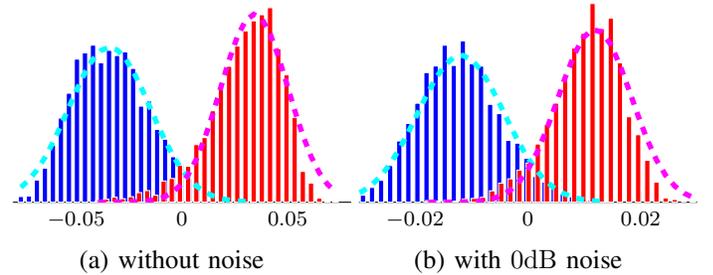

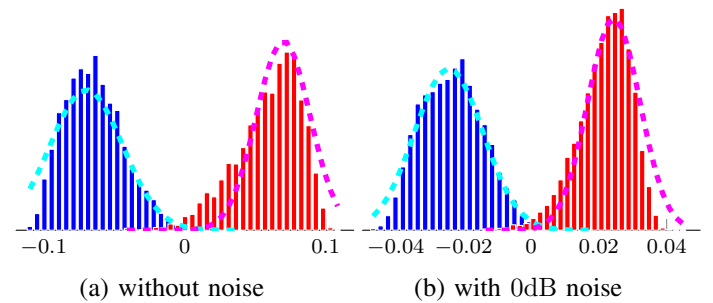
\begin{figure}[htb]
\begin{minipage}[b]{0.48\linewidth}%
\centering{}
\definecolor{mycolor1}{rgb}{0.00000,1.00000,1.00000}%
\definecolor{mycolor2}{rgb}{1.00000,0.00000,1.00000}%
\begin{tikzpicture}[font=\footnotesize]
\begin{axis}[%
width=4.5cm,
height=3cm,
scale only axis,
scaled x ticks = false,
xmin=-0.12,
xmax=0.12,
x tick label style={/pgf/number format/fixed},
every outer y axis line/.append style={white},
every y tick label/.append style={font=\color{white}},
ymin=0,
ymax=25,
ytick={\empty},
axis background/.style={fill=white},
axis x line*=bottom,
axis y line*=left,
]
\addplot[ybar,bar width=2,draw=white,fill=blue,area legend] coordinates{
(-0.109873,0.423375)(-0.104706,1.844707)(-0.099539,5.685327)(-0.094372,8.981607)(-0.089205,13.033915)(-0.084039,14.485487)(-0.078872,16.874534)(-0.073705,17.902732)(-0.068538,17.388633)(-0.063371,19.384546)(-0.058205,16.178989)(-0.053038,14.092353)(-0.047871,13.275843)(-0.042704,9.314259)(-0.037537,6.804248)(-0.032370,5.473639)(-0.027204,4.445442)(-0.022037,3.387003)(-0.016870,1.874948)(-0.011703,1.300367)(-0.006536,0.604822)(-0.001370,0.120964)(0.003797,0.060482)(0.008964,0.060482)(0.014131,0.120964)(0.019298,0.211688)(0.024464,0.090723)(0.029631,0.060482)(0.034798,0.030241)(0.039965,0.030241)
};
\addplot[ybar,bar width=2,draw=white,fill=red,area legend] coordinates{
(-0.040972,0.212401)(-0.035822,0.091029)(-0.030673,0.030343)(-0.025523,0.091029)(-0.020374,0.303431)(-0.015224,0.273088)(-0.010075,0.697890)(-0.004926,0.546175)(0.000224,1.517153)(0.005373,1.759897)(0.010523,3.034306)(0.015672,4.156999)(0.020822,3.701853)(0.025971,5.583123)(0.031121,7.403707)(0.036270,7.343020)(0.041419,9.163604)(0.046569,11.682078)(0.051718,13.229574)(0.056868,15.323245)(0.062017,15.201873)(0.067167,18.509266)(0.072316,19.753332)(0.077466,19.449901)(0.082615,14.443296)(0.087764,11.591049)(0.092914,6.341699)(0.098063,2.457788)(0.103213,0.273088)(0.108362,0.030343)
};
\addplot [color=mycolor1,dashed,line width=2.0pt,forget plot] coordinates{
(-0.109873,4.583486)(-0.104706,6.154070)(-0.099539,7.934181)(-0.094372,9.822337)(-0.089205,11.676177)(-0.084039,13.327834)(-0.078872,14.608028)(-0.073705,15.374346)(-0.068538,15.537274)(-0.063371,15.077387)(-0.058205,14.049162)(-0.053038,12.570364)(-0.047871,10.799865)(-0.042704,8.909676)(-0.037537,7.057950)(-0.032370,5.368690)(-0.027204,3.921310)(-0.022037,2.750218)(-0.016870,1.852150)(-0.011703,1.197728)(-0.006536,0.743727)(-0.001370,0.443447)(0.003797,0.253889)(0.008964,0.139578)(0.014131,0.073683)(0.019298,0.037350)(0.024464,0.018179)(0.029631,0.008497)(0.034798,0.003813)(0.039965,0.001643)
};
\addplot [color=mycolor2,dashed,line width=2.0pt,forget plot] coordinates{
(-0.040972,0.000001)(-0.035822,0.000004)(-0.030673,0.000018)(-0.025523,0.000073)(-0.020374,0.000272)(-0.015224,0.000950)(-0.010075,0.003076)(-0.004926,0.009259)(0.000224,0.025895)(0.005373,0.067295)(0.010523,0.162500)(0.015672,0.364620)(0.020822,0.760222)(0.025971,1.472836)(0.031121,2.651439)(0.036270,4.435296)(0.041419,6.894094)(0.046569,9.957389)(0.051718,13.363719)(0.056868,16.665674)(0.062017,19.312213)(0.067167,20.794805)(0.072316,20.806129)(0.077466,19.343781)(0.082615,16.711102)(0.087764,13.414746)(0.092914,10.006298)(0.098063,6.935505)(0.103213,4.466798)(0.108362,2.673180)
};
\end{axis}
\end{tikzpicture}
\centerline{(a) without noise}\medskip{}
\end{minipage}\hfill{}%
\begin{minipage}[b]{0.48\linewidth}%
\centering{}
\definecolor{mycolor1}{rgb}{0.00000,1.00000,1.00000}%
\definecolor{mycolor2}{rgb}{1.00000,0.00000,1.00000}%
\begin{tikzpicture}[font=\footnotesize]
\begin{axis}[%
width=4.5cm,
height=3cm,
scale only axis,
scaled x ticks = false,
xmin=-0.05,
xmax=0.05,
x tick label style={/pgf/number format/fixed},
every outer y axis line/.append style={white},
every y tick label/.append style={font=\color{white}},
ymin=0,
ymax=55,
ytick={\empty},
axis background/.style={fill=white},
axis x line*=bottom,
axis y line*=left,
]
\addplot[ybar,bar width=2,draw=white,fill=blue,area legend] coordinates{
(-0.046707,0.426942)(-0.044511,1.067356)(-0.042315,3.557854)(-0.040119,8.111907)(-0.037923,15.441086)(-0.035727,23.339522)(-0.033532,30.811015)(-0.031336,33.016884)(-0.029140,33.799612)(-0.026944,35.863167)(-0.024748,36.788209)(-0.022552,39.563336)(-0.020357,41.840362)(-0.018161,35.365068)(-0.015965,29.530187)(-0.013769,25.758862)(-0.011573,20.920181)(-0.009377,14.089101)(-0.007181,10.317776)(-0.004986,7.400336)(-0.002790,3.557854)(-0.000594,2.063555)(0.001602,0.996199)(0.003798,0.284628)(0.005994,0.284628)(0.008189,0.640414)(0.010385,0.142314)(0.012581,0.213471)(0.014777,0.142314)(0.016973,0.071157)
};
\addplot[ybar,bar width=2,draw=white,fill=red,area legend] coordinates{
(-0.013436,0.311118)(-0.011427,0.155559)(-0.009418,0.466677)(-0.007409,0.544456)(-0.005401,1.322251)(-0.003392,1.166692)(-0.001383,1.633369)(0.000626,2.877840)(0.002635,4.433430)(0.004644,6.066798)(0.006653,8.711300)(0.008662,12.289156)(0.010671,17.811498)(0.012679,21.233794)(0.014688,26.678357)(0.016697,33.522950)(0.018706,40.523102)(0.020715,45.812105)(0.022724,52.345580)(0.024733,53.123375)(0.026742,54.134508)(0.028751,44.256516)(0.030759,33.678509)(0.032768,19.444866)(0.034777,10.655787)(0.036786,3.655635)(0.038795,0.544456)(0.040804,0.233338)(0.042813,0.077779)(0.044822,0.077779)
};
\addplot [color=mycolor1,dashed,line width=2.0pt,forget plot] coordinates{
(-0.046707,3.458433)(-0.044511,5.451469)(-0.042315,8.197977)(-0.040119,11.761397)(-0.037923,16.097933)(-0.035727,21.020369)(-0.033532,26.186025)(-0.031336,31.121307)(-0.029140,35.286216)(-0.026944,38.169052)(-0.024748,39.389155)(-0.022552,38.779389)(-0.020357,36.423718)(-0.018161,32.638231)(-0.015965,27.901525)(-0.013769,22.755601)(-0.011573,17.705478)(-0.009377,13.142744)(-0.007181,9.307294)(-0.004986,6.288106)(-0.002790,4.052987)(-0.000594,2.492239)(0.001602,1.462053)(0.003798,0.818268)(0.005994,0.436905)(0.008189,0.222555)(0.010385,0.108155)(0.012581,0.050144)(0.014777,0.022179)(0.016973,0.009359)
};
\addplot [color=mycolor2,dashed,line width=2.0pt,forget plot] coordinates{
(-0.013436,0.000319)(-0.011427,0.001104)(-0.009418,0.003567)(-0.007409,0.010772)(-0.005401,0.030399)(-0.003392,0.080169)(-0.001383,0.197585)(0.000626,0.455087)(0.002635,0.979565)(0.004644,1.970470)(0.006653,3.704281)(0.008662,6.507817)(0.010671,10.684749)(0.012679,16.394221)(0.014688,23.507949)(0.016697,31.501859)(0.018706,39.450739)(0.020715,46.171245)(0.022724,50.499320)(0.024733,51.617501)(0.026742,49.306697)(0.028751,44.016179)(0.030759,36.721150)(0.032768,28.629759)(0.034777,20.860111)(0.036786,14.204077)(0.038795,9.038720)(0.040804,5.375246)(0.042813,2.987358)(0.044822,1.551578)
};
\end{axis}
\end{tikzpicture}
\centerline{(b) with $0\db$ noise}\medskip{}
\end{minipage}
\caption{Gaussian approximation of $g({\bx})$, $n=256,\ p=784,\ c_{1}=c_{2}=\frac{1}{2},\gamma=1$,
Gaussian kernel with $\sigma=1$, Fashion-MNIST data (numbers $8$ and $9$) without and with $0\db$ noise.\label{fig:Gaussian-approximation-Fashion-MINIST} }
\end{figure}

In Figure~\ref{fig:MNIST_threshold_loop} we plot the misclassification rate as a function of the decision threshold $\xi$ for MNIST and Fashion-MNIST data (number $8$ and $9$). We observe that although derived from a Gaussian mixture model, the conclusion from Remark~\ref{rem:Dominant Bias}, Lemma~\ref{lem:norml-y} and Corollary~\ref{cor:Gaussian Approximation normal} that the decision threshold should approximately be $c_2 - c_1$ rather than $0$ approximately holds true in both cases.

\begin{figure}[htb]
\begin{minipage}[b]{0.48\linewidth}%
\centering{}
\begin{tikzpicture}[font=\footnotesize]
\pgfplotsset{every major grid/.append style={densely dashed}}
/xlabel near ticks
/ylabel near ticks
\begin{axis}[%
width=5.3cm,
height=4.5cm,
xmin=0.03,
xmax=0.23,
xminorticks=true,
xlabel={Decision threshold $\xi$},
ymin=0,
ymax=0.6,
ylabel={Classification error},
ytick = {0.1,0.2,0.3,0.4,0.5,0.6},
yticklabels={$0.1$,$0.2$,$0.3$,$0.4$,$0.5$,$0.6$},
grid=major,
axis background/.style={fill=white},
legend style={at={(0.98,0.95)},anchor=north east,legend cell align=left,align=left,draw=white!15!black}
]
\addplot [color=blue,smooth,solid,line width=1pt,mark size=1.5pt,mark=o]
  coordinates{ (0.000000,0.437500)(0.005000,0.437500)(0.010000,0.437500)(0.015000,0.437500)(0.020000,0.437500)(0.025000,0.437500)(0.030000,0.437500)(0.035000,0.437500)(0.040000,0.437500)(0.045000,0.437500)(0.050000,0.437435)(0.055000,0.436003)(0.060000,0.432292)(0.065000,0.423503)(0.070000,0.403711)(0.075000,0.374544)(0.080000,0.333398)(0.085000,0.284180)(0.090000,0.234961)(0.095000,0.195052)(0.100000,0.156120)(0.105000,0.121029)(0.110000,0.095638)(0.115000,0.093359)(0.120000,0.092057)(0.125000,0.096940)(0.130000,0.108529)(0.135000,0.138607)(0.140000,0.171484)(0.145000,0.225781)(0.150000,0.280404)(0.155000,0.344336)(0.160000,0.383789)(0.165000,0.451302)(0.170000,0.496549)(0.175000,0.526888)(0.180000,0.549805)(0.185000,0.558333)(0.190000,0.562174)(0.195000,0.562500)(0.200000,0.562500)(0.205000,0.562500)(0.210000,0.562500)(0.215000,0.562500)(0.220000,0.562500)(0.225000,0.562500)(0.230000,0.562500)
};
\draw[densely dashed,thick,color=red,line width=1.5pt] (axis cs:.12,0) -- (axis cs:.12,.6) ;
\draw[->,thick,color=red] (axis cs:.1,.5) -- (axis cs:.12,.5) node [left,pos=0,font=\footnotesize] {\textcolor{red}{$\xi_{\rm opt} $ }}; 
\draw[densely dashed,thick,color=mycolor2,line width=1.5pt] (axis cs:.125,0) -- (axis cs:.125,.6) ;
\draw[->,thick,color=mycolor2] (axis cs:.14,.06) -- (axis cs:.125,.06) node [right,pos=0,font=\footnotesize] {\textcolor{mycolor2}{$\xi = c_2 - c_1$ }};
\end{axis}
\end{tikzpicture}
\end{minipage}\hfill{}%
\begin{minipage}[b]{0.48\linewidth}%
\centering{}
\begin{tikzpicture}[font=\footnotesize]
\pgfplotsset{every major grid/.append style={densely dashed}}
/xlabel near ticks
/ylabel near ticks
\begin{axis}[%
width=5.3cm,
height=4.5cm,
xmin=0.03,
xmax=0.23,
xminorticks=true,
xlabel={Decision threshold $\xi$},
ymin=-0.05,
ymax=0.6,
ytick = {},
yticklabels={},
grid=major,
axis background/.style={fill=white},
legend style={at={(0.98,0.95)},anchor=north east,legend cell align=left,align=left,draw=white!15!black}
]
\addplot [color=blue,smooth,solid,line width=1pt,mark size=1.5pt,mark=o]
  coordinates{ (0.000000,0.437500)(0.005000,0.437500)(0.010000,0.437435)(0.015000,0.437240)(0.020000,0.434635)(0.025000,0.426367)(0.030000,0.408724)(0.035000,0.380143)(0.040000,0.347526)(0.045000,0.313346)(0.050000,0.268880)(0.055000,0.235612)(0.060000,0.197070)(0.065000,0.156250)(0.070000,0.122591)(0.075000,0.095703)(0.080000,0.065495)(0.085000,0.048242)(0.090000,0.033333)(0.095000,0.019401)(0.100000,0.012174)(0.105000,0.009375)(0.110000,0.006836)(0.115000,0.007617)(0.120000,0.007878)(0.125000,0.012891)(0.130000,0.017448)(0.135000,0.028646)(0.140000,0.036393)(0.145000,0.053255)(0.150000,0.066536)(0.155000,0.086979)(0.160000,0.108919)(0.165000,0.138542)(0.170000,0.178516)(0.175000,0.230664)(0.180000,0.271159)(0.185000,0.323503)(0.190000,0.389648)(0.195000,0.444987)(0.200000,0.494987)(0.205000,0.530599)(0.210000,0.555729)(0.215000,0.561393)(0.220000,0.562500)(0.225000,0.562500)(0.230000,0.562500)
};
\draw[densely dashed,thick,color=red,line width=1.5pt] (axis cs:.11,-.05) -- (axis cs:.11,.6) ;
\draw[->,thick,color=red] (axis cs:.1,.5) -- (axis cs:.11,.5) node [left,pos=0,font=\footnotesize] {\textcolor{red}{$\xi_{\rm opt} $ }}; 
\draw[densely dashed,thick,color=mycolor2,line width=1.5pt] (axis cs:.125,-.05) -- (axis cs:.125,.6) ;
\draw[->,thick,color=mycolor2] (axis cs:.14,-.003) -- (axis cs:.125,-.003) node [right,pos=0,font=\footnotesize] {\textcolor{mycolor2}{$\xi = c_2 - c_1$ }};
\end{axis}
\end{tikzpicture}
\end{minipage}
\caption{$n=512,\ p=784,\ c_{2}-c_1 = 0.125, \gamma=1$, Gaussian kernel with $\sigma=1$ for MNIST (left) and Fashion-MNIST data (right). With optimal decision threshold $\xi_{\rm opt} = 0.12$ (left) and $0.11$ (right) in red. }\label{fig:MNIST_threshold_loop}
\end{figure}
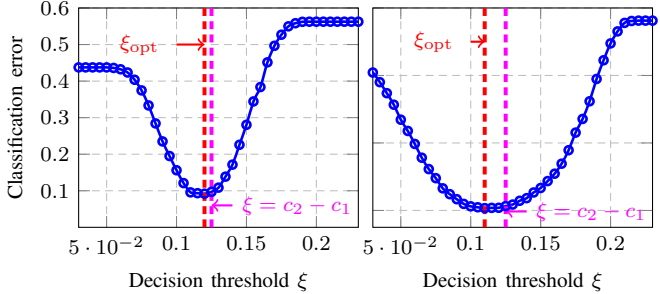

In Figure~\ref{fig:MNIST_sigma_loop} and~\ref{fig:Fashion-MNIST_sigma_loop} we evaluated the performance of LS-SVM on the MNIST and Fashion-MNIST datasets (with and without noise) as a function of the kernel parameter $\sigma$ of Gaussian kernel $f(x)=\exp(-x/2\sigma^2)$. Surprisingly, compared to Figure~\ref{fig:gauss_sigma_loop}, we face the situation where there is little difference in the performance of LS-SVM as soon as $\sigma^2$ is away from $0$, which likely comes from the fact that the difference in means $\bmu_{2}-\bmu_{1}$ is so large that it becomes predominant over the influence of covariances as mentioned in the first paragraph of Section~\ref{subsec:limiting-cases}. This argument is numerically sustained by Table \ref{tab:Empirical-estimation-MNIST}. The gap between theory and practice observed as $\sigma^2\to 0$ is likely a result of the finite $n,p$ (as in Figure~\ref{fig:gauss_sigma_loop}) rather than of the Gaussian assumption of the input data, since we observe a similar behavior even when Gaussian white noise is added.


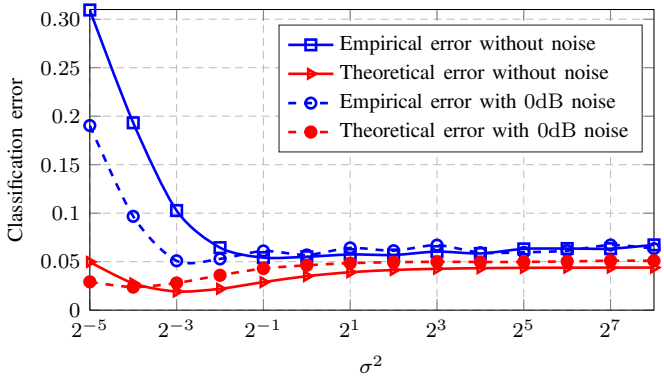
\begin{figure}[htb]
\noindent\begin{minipage}[b]{1\linewidth}%
\centering{}
\begin{tikzpicture}[font=\footnotesize]
\pgfplotsset{every major grid/.append style={densely dashed}}
/xlabel near ticks
/ylabel near ticks
\begin{axis}[%
width=7.5cm,
height=4cm,
scale only axis,
xmode=log,
log basis x={2},
xmin=0.03125,
xmax=256,
xminorticks=true,
xlabel={$\sigma^2$},
ymin=0,
ymax=0.31,
ylabel={Classification error},
ytick = {0,0.05,0.1,0.15,0.2,0.25,0.30},
yticklabels={$0$,$0.05$,$0.1$,$0.15$,$0.2$,$0.25$,$0.30$},
grid=major,
axis background/.style={fill=white},
legend style={at={(0.98,0.95)},anchor=north east,legend cell align=left,align=left,draw=white!15!black}
]
\addplot [color=blue,smooth,solid,line width=1.0pt,mark size=2.0pt,mark=square,mark options={solid}]
  table[row sep=crcr]{%
0.03125	0.309505208333333\\
0.0625	0.193229166666667\\
0.125	0.102734375\\
0.25	0.064453125\\
0.5	0.0541666666666667\\
1	0.055078125\\
2	0.057421875\\
4	0.0569010416666667\\
8	0.0604166666666667\\
16	0.0584635416666667\\
32	0.0631510416666667\\
64	0.0635416666666667\\
128	0.0634114583333333\\
256	0.0673177083333333\\
};
\addlegendentry{Empirical error without noise};

\addplot [color=red,smooth,solid,line width=1.0pt,mark size=2.0pt,mark=triangle,mark options={solid,rotate=270}]
  table[row sep=crcr]{%
0.03125	0.0497040761023577\\
0.0625	0.0276483828889777\\
0.125	0.0192351044541105\\
0.25	0.0220228166039465\\
0.5	0.028907288280632\\
1	0.0350707246661068\\
2	0.0391205621635671\\
4	0.0414319154558093\\
8	0.0426651639057672\\
16	0.0433019441865176\\
32	0.0436254694959261\\
64	0.0437885279981571\\
128	0.0438703827199064\\
256	0.0439113916371357\\
};
\addlegendentry{Theoretical error without noise};

\addplot [color=blue,smooth,dashed,line width=1.0pt,mark size=2.0pt,mark=o,mark options={solid}]
  table[row sep=crcr]{%
0.03125	0.190364583333333\\
0.0625	0.0967447916666667\\
0.125	0.0509114583333333\\
0.25	0.0528645833333333\\
0.5	0.0610677083333333\\
1	0.0569010416666667\\
2	0.0641927083333333\\
4	0.061328125\\
8	0.0670572916666667\\
16	0.0592447916666667\\
32	0.0596354166666667\\
64	0.061328125\\
128	0.0671875\\
256	0.0634114583333333\\
};
\addlegendentry{Empirical error with $0\db$ noise};

\addplot [color=red,smooth,dashed,line width=1.0pt,mark size=2.0pt,mark=*,mark options={solid}]
  table[row sep=crcr]{%
0.03125	0.0291804816884539\\
0.0625	0.0237575193140776\\
0.125	0.0280293526313407\\
0.25	0.0359578979940397\\
0.5	0.0429300405034688\\
1	0.0462731613202686\\
2	0.0485067191840601\\
4	0.0494038051144224\\
8	0.04985940949061\\
16	0.0496946739818815\\
32	0.0497289461864644\\
64	0.0502670565062374\\
128	0.0510058266539532\\
256	0.0509217853468831\\
};
\addlegendentry{Theoretical error with $0\db$ noise};

\end{axis}
\end{tikzpicture}
\caption{$n=256,\ p=784,\ c_{1}=c_{2}=1/2,\gamma=1$, Gaussian kernel, MNIST data (numbers $8$ and $9$) with and without noise.}\label{fig:MNIST_sigma_loop}
\end{minipage}
\end{figure}

\begin{figure}[htb]
\noindent\begin{minipage}[b]{1\linewidth}%
\centering{}
\begin{tikzpicture}[font=\footnotesize]
\pgfplotsset{every major grid/.append style={densely dashed}}
/xlabel near ticks
/ylabel near ticks
\begin{axis}[%
width=7.5cm,
height=4cm,
scale only axis,
xmode=log,
log basis x={2},
xmin=0.03125,
xmax=256,
xminorticks=true,
xlabel={$\sigma^2$},
ymin=-0.01,
ymax=0.27,
ylabel={Classification error},
ytick = {0,0.05,0.1,0.15,0.2,0.25,0.30},
yticklabels={$0$,$0.05$,$0.1$,$0.15$,$0.2$,$0.25$,$0.30$},
grid=major,
axis background/.style={fill=white},
legend style={at={(0.98,0.95)},anchor=north east,legend cell align=left,align=left,draw=white!15!black}
]
\addplot [color=blue,smooth,solid,line width=1.0pt,mark size=2.0pt,mark=square,mark options={solid}]
  coordinates{ (0.031250,0.267318)(0.062500,0.193359)(0.125000,0.117708)(0.250000,0.039323)(0.500000,0.013542)(1.000000,0.006510)(2.000000,0.007031)(4.000000,0.005599)(8.000000,0.006250)(16.000000,0.005859)(32.000000,0.006120)(64.000000,0.009115)(128.000000,0.006250)(256.000000,0.007422)
};
\addlegendentry{Empirical error without noise};
\addplot [color=red,smooth,solid,line width=1.0pt,mark size=2.0pt,mark=triangle,mark options={solid,rotate=270}]
  coordinates{ (0.031250,0.068979)(0.062500,0.042419)(0.125000,0.017988)(0.250000,0.005925)(0.500000,0.003081)(1.000000,0.003114)(2.000000,0.003780)(4.000000,0.004392)(8.000000,0.004795)(16.000000,0.005025)(32.000000,0.005147)(64.000000,0.005211)(128.000000,0.005243)(256.000000,0.005259)
};
\addlegendentry{Theoretical error without noise};
\addplot [color=blue,smooth,dashed,line width=1.0pt,mark size=2.0pt,mark=o,mark options={solid}]
  coordinates{ (0.031250,0.189453)(0.062500,0.101042)(0.125000,0.033203)(0.250000,0.011589)(0.500000,0.007812)(1.000000,0.007682)(2.000000,0.008724)(4.000000,0.008073)(8.000000,0.006771)(16.000000,0.006510)(32.000000,0.007682)(64.000000,0.011328)(128.000000,0.007812)(256.000000,0.007292)
};
\addlegendentry{Empirical error with $0\db$ noise};
\addplot [color=red,smooth,dashed,line width=1.0pt,mark size=2.0pt,mark=*,mark options={solid}]
  coordinates{ (0.031250,0.037421)(0.062500,0.015659)(0.125000,0.005196)(0.250000,0.003515)(0.500000,0.004697)(1.000000,0.004951)(2.000000,0.005810)(4.000000,0.006673)(8.000000,0.006222)(16.000000,0.006458)(32.000000,0.007022)(64.000000,0.006859)(128.000000,0.006514)(256.000000,0.006853)
};
\addlegendentry{Theoretical error with $0\db$ noise};

\end{axis}
\end{tikzpicture}
\caption{$n=256,\ p=784,\ c_{1}=c_{2}=1/2,\gamma=1$, Gaussian kernel, Fashion-MNIST data (numbers $8$ and $9$) with and without noise.}\label{fig:Fashion-MNIST_sigma_loop}
\end{minipage}
\end{figure}

\begin{table}[htb]
\centering
\caption{Empirical estimation of differences in means and covariances of MNIST/Fashion-MNIST data (numbers $8$ and $9$)}
\label{tab:Empirical-estimation-MNIST}
\begin{tabular}{c|c|c}
  & \makecell{MNIST/Fashion-MNIST\\ without noise} & \makecell{MNIST/Fashion-MNIST \\ with $0\db$ noise} \\ 
  \hline
$\|\bmu_{2}-\bmu_{2}\|^{2}$ & $251/483$           & $96/197$         \\
$\frac1p\left(\tr\left({\bf C}_{2}-{\bf C}_{1}\right)\right)^{2}$ & $19/89$            & $3/13$          \\
$\frac1p\tr\left((\mathbf{C}_{2}-\mathbf{C}_{1})^{2}\right)$ & $30/86$            & $5/13$         
\end{tabular}
\end{table}

\section{Concluding remarks}
\label{sec:conclusion}
In this work, through a performance analysis of LS-SVM for large dimensional data, we reveal the significance of balanced dataset with $c_1=c_2$, as well as the interplay between the pivotal kernel function $f$ and the statistical structure of the data. The normalized labels $y^*_i\in\{-1/c_1,1/c_2\}$ are proposed to mitigate the damage of $c_2-c_1$ in the decision function. We prove the irrelevance of $\gamma$ when it is considered to remain constant in the large $n,p$ regime; however, this argument is not guaranteed to hold true when $\gamma$ scales with $n,p$. Our theoretical results, even though built upon the assumption of Gaussian data, provide similar results when tested on real-world large dimensional datasets, which offers a possible application despite the strong Gaussian assumption in the general context of large scale supervised learning.

The major difference of the present work compared to other theoretical analyses (for example \citep{caponnetto2007optimal}) is that, by studying the rather simple problem of a two-class Gaussian mixture separation with comparably large instance number and data dimension, together with sufficiently smooth kernel function $f$ and regularization parameter $\gamma$ of order $O(1)$, we deduce \emph{explicit} results for the output of LS-SVM which surprisingly coincide with observations on some large dimensional real-world datasets (including MNIST and beyond) and therefore allowing for novel insights into the behavior of LS-SVM for large dimensional datasets. Of interest to future work is the remark that, unlike in the work of \citep{caponnetto2007optimal} where, in the large $n$ alone asymptotics, $\gamma$ is best scaled large with $n$, in the present large $p$, large $n$ setting, where we demonstrate rate-optimality of LS-SVM for $\gamma=O(1)$. This apparent paradox could be deciphered through the analysis of more advanced (normalized inner product) kernels of the type $f(\bx_i^\T \bx_j/\sqrt{p})$, studied notably in \citep{cheng2013spectrum}, for which we believe that other scalings for $\gamma$ would be optimal; it is also importantly believed that such kernels could lead to improved performances (not in rate, as those are already optimal in the present setting, but possibly in absolute performance). These technically more involved considerations are left for future investigations.

The extension of the present work to the asymptotic performance analysis of the classical SVM requires more efforts since, there, the decision function $g(\bx)$ depends implicitly (through the solution to a quadratic programming problem) rather than explicitly on the underlying kernel matrix $\bK$. Additional technical tools are thus required to cope with this dependence structure.

The link between LS-SVM and extreme learning machine (ELM) was brought to light in \cite{huang2012extreme} and the performance analysis of ELM in large dimension has been investigated in the recent article \cite{louart2018random}. Together with these works, we have the possibility to identify the tight but subtle relation between the kernel function and the activation function in the context of some simple structured neural networks. This is notably of interest when the datasets are so large that computing $\bK$ and the decision function $g(\bx)$ becomes prohibitive, a problem largely alleviated by neural networks with controllable number of neurons. This link also generally opens up a possible direction of research into the complex neural networks realm.

{\footnotesize
\bibliographystyle{IEEEtran}
\bibliography{IEEEabrv,RMT4LSSVM}
}

\clearpage

\begin{center}
  {\Large \textbf{Supplementary Material\\}} \vskip 0.1in \textbf{A Large Dimensional Analysis of \\Least Squares Support Vector Machines}
\end{center}
\vskip 0.3in




\appendices

\section{Proof of Theorem \ref{thm:Random Equivalent} }
\label{app:proof-theo-1}

Our key interest here is on the decision function of LS-SVM: $g(\bx)=\balpha^{\T}\bk(\bx)+b$ with $(\balpha,b)$ given by
\begin{equation*}
\begin{cases}
\balpha & ={\bS}^{-1}\left(\iden-\frac{\mathbf{1}_n\mathbf{1}_n^\T\bS^{-1}}{\mathbf{1}_n^{\T}\bS^{-1}\mathbf{1}_n}\right)\by\\
b & =\frac{{\bf 1}_{n}^{\T}{\bS}^{-1}\by}{{\bf 1}_{n}^{\T}{\bS}^{-1}{\bf 1}_{n}}
\end{cases}
\end{equation*}
and $\bS^{-1}=\left({\bK}+\frac{n}{\gamma}\iden\right)^{-1}$. 

Before going into the detailed proof, as we will frequently deal with random variables evolving as $n,p$ grow large, we shall use the extension of the $O(\cdot)$ notation introduced in \cite{couillet2016kernel}: for a random variable $x\equiv x_n$ and $u_n\ge 0$, we write $x=O(u_n)$ if for any $\eta>0$ and $D>0$, we have $n^D{\rm P}(x\ge n^\eta u_n)\to 0$. Note that under Assumption~\ref{as:Growth rate} it is equivalent to use either $O(u_n)$ or $O(u_p)$ since $n,p$ scales linearly. In the following we shall use constantly $O(u_n)$ for simplicity.

When multidimensional objects are concerned, $\bv=O(u_n)$ means the maximum entry of a vector (or a diagonal matrix) $\bv$ in absolute value is of order $O(u_n)$ and $\bM=O(u_n)$ means that the operator norm of $\bM$ is of order $O(u_n)$. We refer the reader to \cite{couillet2016kernel} for more discussions on these practical definitions.

Under the growth rate settings of Assumption~\ref{as:Growth rate}, from \cite{couillet2016kernel}, the approximation of the kernel matrix $\bK$ is given by
\begin{equation}
\bK = -2\fftau\left(\bP\bOmega^{\T}\bOmega\bP+\bA\right)+\beta\iden+O(n^{-\frac{1}{2}})
\end{equation}
with $\beta=f(0)-\ftau+\tau\fftau$ and $\bA = \bA_n+\bA_{\sqrt{n}}+\bA_1$, $\bA_n=-\frac{\ftau}{2\fftau}\mathbf{1}_n\mathbf{1}_n^{\T}$ and $\bA_{\sqrt{n}}$, $\bA_1$ given by \eqref{eq:VAV2} and \eqref{eq:VAV3} at the top of next page, where we denote 
\begin{align*}
t_a &\triangleq \frac{\tr(\bC_a-\bC^\circ)}{\sqrt{p}}=O(1)\\
(\bpsi)^2 &\triangleq[(\bpsi_1)^2,\ldots,(\bpsi_n)^2]^{\sf T}.
\end{align*}

\begin{figure*}[!t]
\normalsize
\setcounter{MYtempeqncnt}{\value{equation}}
\setcounter{equation}{17}
\begin{align}
\label{eq:VAV2}
\bA_{\sqrt{n}}&=-\frac{1}{2}\left[\bpsi\mathbf{1}_n^{\T}+\mathbf{1}_n\bpsi^{\T}+\left\{t_a\frac{\mathbf{1}_{n_a}}{\sqrt{p}}\right\}_{a=1}^2\mathbf{1}_n^{\T}+\mathbf{1}_n\left\{t_b\frac{\mathbf{1}_{n_b}^{\T}}{\sqrt{p}}\right\}_{b=1}^2\right]\\
\label{eq:VAV3}
\bA_1&=-\frac{1}{2}\left[ \left\{\|\bmu_a-\bmu_b\|^2\frac{\mathbf{1}_{n_a}\mathbf{1}_{n_b}^{\T}}{p}\right\}_{a,b=1}^2
+2\left\{\frac{(\bOmega\bP)_a^{\T}(\bmu_b-\bmu_a)\mathbf{1}_{n_b}^{\T}}{\sqrt{p}} \right\}_{a,b=1}^2 
-2\left\{\frac{\mathbf{1}_{n_a}(\bmu_b-\bmu_a)^{\T}(\bOmega\bP)_b}{\sqrt{p}} \right\}_{a,b=1}^2\right] \nonumber \\
& - \frac{\ffftau}{4\fftau}\left[(\bpsi)^2\mathbf{1}_n^{\T}+\mathbf{1}_n[(\bpsi)^2]^{\T}+\left\{t_a^2\frac{\mathbf{1}_{n_a}}{p} \right\}_{a=1}^2\mathbf{1}_n^{\T}+\mathbf{1}_n\left\{t_b^2\frac{\mathbf{1}_{n_b}^{\T}}{p} \right\}_{b=1}^2 +2\left\{t_a t_b\frac{\mathbf{1}_{n_a}\mathbf{1}_{n_b}^{\T}}{p} \right\}_{a,b=1}^2 + 2\mathcal{D}\{t_a \mathbf{I}_{n_a}\}_{a=1}^2\bpsi\frac{\mathbf{1}_n^{\T}}{\sqrt{p}} \right. \nonumber\\
&+\left.2\bpsi\left\{t_b\frac{\mathbf{1}_{n_b}^{\T}}{\sqrt{p}}\right\}_{b=1}^2+ 2\frac{\mathbf{1}_n}{\sqrt{p}}(\bpsi)^{\T}\mathcal{D}\{t_a \mathbf{1}_{n_a}\}_{a=1}^2 +2\left\{t_a\frac{\mathbf{1}_{n_a}}{\sqrt{p}}\right\} _{a=1}^2(\bpsi)^{\T}+ 4\left\{\tr(\bC_a\bC_b)\frac{\mathbf{1}_{n_a}\mathbf{1}_{n_b}^{\T}}{p^2} \right\}_{a,b=1}^2 +2\bpsi(\bpsi)^{\T} \right]\\
\label{eq:tilde-k}
\tilde{\bk}(\bx)&=\fftau\bigg[\left\{\frac{\|\bmu_b-\bmu_a\|^2}{p}\mathbf{1}_{n_b}\right\}_{b=1}^2-\frac{2}{\sqrt{p}}\left\{\mathbf{1}_{n_b}(\bmu_b-\bmu_a)^{\T}\right\}_{b=1}^2\bomega_\bx+\frac{2}{\sqrt{p}}\mathcal{D}\left(\left\{\mathbf{1}_{n_b}(\bmu_b-\bmu_a)^{\T}\right\}_{b=1}^2\bOmega\right) \bigg] \nonumber\nonumber\\
&+\frac{\ffftau}{2}\bigg[\left\{\frac{(t_a+t_b)^2}{p}\mathbf{1}_{n_b}\right\}_{b=1}^2 +2\mathcal{D}\bigg(\left\{\frac{t_a+t_b}{\sqrt{p}}\mathbf{1}_{n_b}\right\}_{b=1}^2\bigg)\bpsi+2\left\{\frac{t_a+t_b}{\sqrt{p}}\mathbf{1}_{n_b}\right\}_{b=1}^2\psi_\bx+(\bpsi)^2+2\psi_\bx\bpsi+\psi_\bx^2\mathbf{1}_n\nonumber\\
&+\bigg\{\frac{4}{p^2}\tr(\bC_a\bC_b)\mathbf{1}_{n_b}\bigg\}_{b=1}^2\bigg]
\end{align}
\setcounter{equation}{\value{MYtempeqncnt}}
\hrulefill
\vspace*{4pt}
\end{figure*}

We start with the term $\bS^{-1}$. The terms of leading order in $\bK$, i.e.,$-2\fftau\bA_n$ and $\frac{n}{\gamma}\iden$ are both of operator norm $O(n)$. Therefore a Taylor expansion can be performed as
\begin{align*}
&\bS^{-1}=\left(\bK+\frac{n}{\gamma}\iden\right)^{-1}=\frac{1}{n}\bigg[\bL^{-1}-\frac{2\fftau}{n}\\
&\left(\bA_{\sqrt{n}}+\bA_1+\bP\bOmega^{\T}\bOmega\bP\right)+\frac{\beta\iden}{n}+O(n^{-\frac{3}{2}})\bigg]^{-1}\\
&=\frac{\bL}{n}+\frac{2\fftau}{n^2}\bL\bA_{\sqrt{n}}\bL+\bL\left(\bQ-\frac{\beta}{n^2}\iden\right)\bL+O(n^{-\frac{5}{2}})
\end{align*}
with $\bL=\left(\ftau\vnones+\frac{\iden}{\gamma}\right)^{-1}$ of order $O(1)$ and 
$\bQ= \frac{2\fftau}{n^2}\left(\bA_1+\bP\bOmega^{\T}\bOmega\bP+\frac{2\fftau}{n}\bA_{\sqrt{n}}\bL\bA_{\sqrt{n}}\right)$.

With the Sherman-Morrison formula we are able to compute explicitly $\bL$ as
\begin{align}
\bL&=\left(\ftau\vnones+\frac{\iden}{\gamma}\right)^{-1}\nonumber=\gamma\left(\iden-\frac{\gamma\ftau}{1+\gamma\ftau}\vnones\right)\nonumber\\
&=\frac{\gamma}{1+\gamma\ftau}\iden+\frac{\gamma^2\ftau}{1+\gamma\ftau}\bP=O(1).\label{eq:L}
\end{align}

Writing $\bL$ as a linear combination of $\iden$ and $\bP$ is useful when computing $\bL\mathbf{1}_n$ or $\mathbf{1}_n^{\T}\bL$, because by the definition of $\bP=\iden-\vnones$, we have $\mathbf{1}_n^{\T}\bP=\bP\mathbf{1}_n=\mathbf{0}$.

We shall start with the term $\mathbf{1}_n^{\T}\bS^{-1}$, since it is the basis of several other terms appearing in $\balpha$ and $b$,
\begin{align*}
\mathbf{1}_n^{\T}\bS^{-1}&=\frac{\gamma\mathbf{1}_n^{\T}}{1+\gamma\ftau}\left[\frac{\iden}{n}+\frac{2\fftau}{n^2}\bA_{\sqrt{n}}\bL+\left(\bQ-\frac{\beta}{n^2}\iden\right)\bL \right]\\
&+O(n^{-\frac{3}{2}})
\end{align*}
since $\mathbf{1}_n^{\T}\bL=\frac{\gamma}{1+\gamma\ftau}\mathbf{1}_n^{\T}$.

With $\mathbf{1}_n^{\T}\bS^{-1}$ at hand, we next obtain,

\begin{align}
\mathbf{1}_n\mathbf{1}_n^{\T}\bS^{-1}&=\frac{\gamma}{1+\gamma\ftau} \bigg[ \underbrace{\frac{\mathbf{1}_n\mathbf{1}_n^{\T}}{n}}_{O(1)} + \underbrace{\frac{2\fftau}{n^2}\mathbf{1}_n\mathbf{1}_n^{\T}\bA_{\sqrt{n}}\bL}_{O(n^{-1/2})} \nonumber\\ 
&+ \underbrace{\mathbf{1}_n\mathbf{1}_n^{\T}\left(\bQ-\frac{\beta}{n^2}\iden\right)\bL}_{O(n^{-1})}  \bigg]+O(n^{-\frac{3}{2}}) \label{eq:1-1-S}
\end{align}
\begin{align*}
\mathbf{1}_n^{\T}\bS^{-1}\by&=\frac{\gamma}{1+\gamma\ftau} \bigg[ \underbrace{c_2-c_1}_{O(1)} + \underbrace{\frac{2\fftau}{n^2}\mathbf{1}_n^{\T}\bA_{\sqrt{n}}\bL\by}_{O(n^{-1/2})}\\
&+ \underbrace{\mathbf{1}_n^{\T}\left(\bQ-\frac{\beta}{n^2}\iden\right)\bL\by}_{O(n^{-1})}  \bigg]+O(n^{-\frac{3}{2}})
\end{align*}
\begin{align*}
\mathbf{1}_n^{\T}\bS^{-1}\mathbf{1}_n&=\frac{\gamma}{1+\gamma\ftau} \bigg[ \underbrace{1}_{O(1)}+  \underbrace{\frac{2\fftau}{n^2}\frac{\gamma\mathbf{1}_n^{\T}\bA_{\sqrt{n}}\mathbf{1}_n}{1+\gamma\ftau}}_{O(n^{-1/2})}\\
&+ \underbrace{\frac{\gamma}{1+\gamma\ftau}\mathbf{1}_n^{\T}\left(\bQ-\frac{\beta}{n^2}\iden\right)\mathbf{1}_n}_{O(n^{-1})}  \bigg]+O(n^{-\frac{3}{2}}).
\end{align*}

The inverse of $\mathbf{1}_n^{\T}\bS^{-1}\mathbf{1}_n$ can consequently be computed using a Taylor expansion around its leading order, allowing an error term of $O(n^{-\frac{3}{2}})$ as
\begin{align}
\frac{1}{\mathbf{1}_n^{\T}\bS^{-1}\mathbf{1}_n}&=\frac{1+\gamma\ftau}{\gamma} \bigg[ \underbrace{1}_{O(1)}-  \underbrace{\frac{2\fftau}{n^2}\frac{\gamma\mathbf{1}_n^{\T}\bA_{\sqrt{n}}\mathbf{1}_n}{1+\gamma\ftau}}_{O(n^{-1/2})} \nonumber\\
&- \underbrace{\frac{\gamma}{1+\gamma\ftau}\mathbf{1}_n^{\T}\left(\bQ-\frac{\beta}{n^2}\iden\right)\mathbf{1}_n}_{O(n^{-1})}  \bigg]+O(n^{-\frac{3}{2}}).\label{eq:inv-1-S-1}
\end{align}

Combing \eqref{eq:1-1-S} with \eqref{eq:inv-1-S-1} we deduce
\begin{align}
&\frac{\mathbf{1}_n\mathbf{1}_n^{\T}\bS^{-1}}{\mathbf{1}_n^{\T}\bS^{-1}\mathbf{1}_n}=\underbrace{\vnones}_{O(1)}+\underbrace{\frac{2\fftau}{n^2}\mathbf{1}_n\mathbf{1}_n^{\T}\bA_{\sqrt{n}}\bigg[\bL-\frac{\gamma\vnones}{1+\gamma\ftau} \bigg]}_{O(n^{-1/2})}\nonumber\\
&+\underbrace{\mathbf{1}_n\mathbf{1}_n^{\T}\left(\bQ-\frac{\beta}{n^2}\iden\right)\bigg[\bL-\frac{\gamma\vnones}{1+\gamma\ftau} \bigg]}_{O(n^{-1})}+O(n^{-\frac{3}{2}})
\end{align}
and similarly the following approximation of $b$ as
\begin{align}
b&=\underbrace{c_2-c_1}_{O(1)}-\underbrace{\frac{2\gamma}{\sqrt{p}}c_1c_2\fftau(t_2-t_1)}_{O(n^{-1/2})} -\underbrace{\frac{\gamma\fftau}{n}\by^{\T}\bP\bpsi}_{O(n^{-1})} \nonumber\\
&\underbrace{-\frac{\gamma\ffftau}{2n}\by^{\T}\bP(\bpsi)^2+\frac{4\gamma c_1c_2}{p}[c_1T_1+(c_2-c_1)D-c_2T_2]}_{O(n^{-1})}\nonumber\\
&+O(n^{-\frac{3}{2}})\label{eq:b}
\end{align} 
where
\begin{align*}
D&=\frac{\fftau}{2}\|\bmu_2-\bmu_1\|^2+\frac{\ffftau}{4}(t_1+t_2)^2+\ffftau\frac{\tr\bC_1\bC_2}{p}\\
T_a&=\ffftau t_a^2+\ffftau\frac{\tr\bC_1\bC_2}{p}
\end{align*}
which gives the asymptotic approximation of $b$.

Moving to $\balpha$, note from \eqref{eq:L} that $\bL-\frac{\gamma}{1+\gamma\ftau}\vnones =\gamma\bP$, and we can thus rewrite:
\begin{align*}
\frac{\mathbf{1}_n\mathbf{1}_n^{\T}\bS^{-1}}{\mathbf{1}_n^{\T}\bS^{-1}\mathbf{1}_n}&=\vnones+\frac{2\gamma\fftau}{n^2}\mathbf{1}_n\mathbf{1}_n^{\T}\bA_{\sqrt{n}}\bP\\
&+\gamma\mathbf{1}_n\mathbf{1}_n^{\T}\left(\bQ-\frac{\beta}{n^2}\iden\right)\bP+O(n^{-\frac{3}{2}}).
\end{align*}

\addtocounter{equation}{3}

At this point, for $\balpha =\bS^{-1}\left(\iden-\frac{\mathbf{1}_n\mathbf{1}_n^{\T}\bS^{-1}}{\mathbf{1}_n^{\T}\bS^{-1}\mathbf{1}_n}\right)\by$, we have
\begin{align*}
\balpha&=\bS^{-1}\bigg[\iden-\frac{2\gamma\fftau}{n^2}\mathbf{1}_n\mathbf{1}_n^{\T}\bA_{\sqrt{n}}\\
&-\gamma\mathbf{1}_n\mathbf{1}_n^{\T}\left(\bQ-\frac{\beta}{n^2}\iden\right)\bigg]\bP\by+O(n^{-\frac{5}{2}}).
\end{align*}

Here again, we use $\mathbf{1}_n^{\T}\bL=\frac{\gamma}{1+\gamma\ftau}\mathbf{1}_n^{\T}$ and $\bL-\frac{\gamma}{1+\gamma\ftau}\vnones=\gamma\bP$, to eventually get
\begin{align}
\balpha&=\underbrace{\frac{\gamma}{n}\bP\by}_{O(n^{-1})}+\underbrace{\gamma^2\bP\left(\bQ-\frac{\beta}{n^2}\iden\right)\bP\by}_{O(n^{-2})}\label{eq:alpha}\\
&-\underbrace{\frac{\gamma^2}{1+\gamma\ftau}\left(\frac{2\fftau}{n^2}\right)^2\bL\bA_{\sqrt{n}}\mathbf{1}_n\mathbf{1}_n^{\T}\bA_{\sqrt{n}}\bP\by}_{O(n^{-2})}+O(n^{-\frac{5}{2}}).\nonumber
\end{align}
Note here the absence of a term of order $O(n^{-3/2})$ in the expression of $\balpha$ since $\bP\bA_{\sqrt{n}}\bP=0$ from \eqref{eq:VAV2}. 

We shall now work on the vector $\bk(\bx)$ for a new datum $\bx$, following the same analysis as in \cite{couillet2016kernel} for the kernel matrix $\bK$, assuming that $\bx\sim\mathcal{N}(\bmu_a,\bC_a)$ and recalling the random variables definitions,
\begin{align*}
\bomega_\bx &\triangleq (\bx-\bmu_a)/\sqrt{p}\\
\psi_\bx &\triangleq \|\bomega_\bx\|^2-\mathbb{E}\|\bomega_\bx\|^2
\end{align*}
we show that the $j$-th entry of $\bk(\bx)$ can be written as
\begin{align}
&[\bk(\bx)]_j=\underbrace{\ftau}_{O(1)}+\fftau\bigg[\underbrace{\frac{t_a+t_b}{\sqrt{p}}+\psi_x+\psi_j-2(\bomega_\bx)^{\T}\bomega_j}_{O(n^{-1/2})}\nonumber\\
&+\underbrace{\frac{\|\bmu_b-\bmu_a\|^2}{p}+\frac{2}{\sqrt{p}}(\bmu_b-\bmu_a)^{\T}(\bomega_j-\bomega_\bx)}_{O(n^{-1})}\bigg]+\frac{\ffftau}{2}\nonumber\\
&\bigg[\underbrace{\left(\frac{t_a+t_b}{\sqrt{p}}+\psi_j+\psi_\bx\right)^2+\frac{4}{p^2}\tr \bC_a\bC_b}_{O(n^{-1})}\bigg]+O(n^{-\frac{3}{2}}).\label{eq:k(x)}
\end{align}

Combining \eqref{eq:alpha} and \eqref{eq:k(x)}, we deduce
\begin{align}
&\balpha^{\T}\bk(\bx)=\underbrace{\frac{2\gamma}{\sqrt{p}}c_1c_2\fftau(t_2-t_1)}_{O(n^{-1/2})}+\underbrace{\frac{\gamma}{n}\by^{\T}\bP\tilde{\bk}(\bx)}_{O(n^{-1})}\nonumber\\
&+\underbrace{\frac{\gamma\fftau}{n}\by^{\T}\bP(\bpsi-2\bP\bOmega^{\T}\bomega_\bx)}_{O(n^{-1})}+O(n^{-\frac{3}{2}})\label{eq:alpha-k(x)-1}
\end{align}
with $\tilde{\bk}(\bx)$ given in \eqref{eq:tilde-k}.

At this point, note that the term of order $O(n^{-\frac{1}{2}})$ in the final object $g(\bx)=\balpha^{\T}\bk(\bx)+b$ disappears because in both \eqref{eq:b} and \eqref{eq:alpha-k(x)-1} the term of order $O(n^{-1/2})$ is $\frac{2\gamma}{\sqrt{p}}c_1c_2\fftau(t_2-t_1)$ but of opposite signs. Also, we see that the leading term $c_2-c_1$ in $b$ will remain in $g(\bx)$ as stated in Remark~\ref{rem:Dominant Bias}.

The development of $\by^{\T}\bP\tilde{\bk}(\bx)$ induces many simplifications, since i) $\bP\mathbf{1}_n=\mathbf{0}$ and ii) random variables as $\bomega_\bx$ and $\bpsi$ in $\tilde{\bk}(\bx)$, once multiplied by $\by^{\T}\bP$, thanks to probabilistic averaging of independent zero-mean terms, are of smaller order and thus become negligible. We thus get
\begin{align}
&\frac{\gamma}{n}\by^{\T}\bP\tilde{\bk}(\bx)=2\gamma c_1c_2\fftau\bigg[ \frac{\|\bmu_2-\bmu_a\|^2-\|\bmu_1-\bmu_a\|^2}{p}\nonumber\\
&-2(\bomega_\bx)^{\T}\frac{\bmu_2-\bmu_1}{\sqrt{p}}\bigg]+\frac{\gamma\ffftau}{2n}\by^{\T}\bP(\bpsi)^2+\gamma c_1c_2\ffftau\bigg[ \nonumber\\
&2\left(\frac{t_a}{\sqrt{p}}+\psi_\bx\right)\frac{t_2-t_1}{\sqrt{p}}+\frac{t_2^2-t_1^2}{p}+\frac{4}{p^2}\tr(\bC_a\bC_2-\bC_a\bC_1) \bigg]\nonumber\\
&+O(n^{-\frac{3}{2}}).\label{eq:alpha-k(x)-n-1}
\end{align}

This result, together with \eqref{eq:alpha-k(x)-1}, completes the analysis of the term $\balpha^{\T}\bk(\bx)$. Combining \eqref{eq:alpha-k(x)-1}-\eqref{eq:alpha-k(x)-n-1} with \eqref{eq:b} we conclude the proof of Theorem~\ref{thm:Random Equivalent}.

\section{Proof of Theorem \ref{thm:Gaussian Approximation}}
\label{app:proof-theo-2}

This section is dedicated to the proof of the central limit theorem for
\[\hat{g}(\bx)=c_2-c_1+\gamma\left(\mathfrak{P}+c_\bx\mathfrak{D}\right)\]
 with the shortcut $c_\bx=-2c_1c_2^2$ for $\bx\in\mathcal{C}_1$ and $c_\bx=2c_1^2c_2$ for $\bx\in\mathcal{C}_2$, and $\mathfrak{P}, \mathfrak{D}$ as defined in \eqref{eq:P} and \eqref{eq:D}.

 Our objective is to show that for $a\in\{1,2\}$, $n(\hat{g}(\bx)-G_a)\cd 0$ with 
 \[
 	G_a\sim\mathcal{N}({\rm E}_a, {\rm Var}_a)
 \] 
where ${\rm E}_a$ and ${\rm Var}_a$ are given in Theorem~\ref{thm:Gaussian Approximation}. We recall that $\bx=\bmu_a+\sqrt{p}\bomega_\bx$ with $\bomega_\bx\sim\mathcal{N}(0,\bC_a/p)$.

Letting $\bz_\bx$ such that $\bomega_\bx=\bC_a^{1/2} \bz_\bx /\sqrt{p}$, we have $ \bz_\bx\sim\mathcal{N}(\mathbf{0},\iden)$ and we can rewrite $\hat{g}(\bx)$ in the following quadratic form (of $\bz_\bx$) as 
\[
\hat{g}(\bx)=\bz_\bx^{\T}\bA\bz_\bx+\bz_\bx^{\T}\bb+c
\]
with
\begin{align*}
\bA&=2\gamma c_1c_2\ffftau\frac{\tr (\bC_2-\bC_1)}{p}\frac{\bC_a}{p}\\
\bb&=-\frac{2\gamma\fftau}{n}\frac{\left(\bC_a\right)^{\frac{1}{2}}}{\sqrt{p}}\bOmega\bP\by-\frac{4c_1 c_2\gamma\fftau}{\sqrt{p}}\frac{\left(\bC_a\right)^{\frac{1}{2}}}{\sqrt{p}}\left(\bmu_2-\bmu_1\right)\\
c&=c_2-c_1+\gamma c_\bx\mathfrak{D}-2\gamma c_1c_2\ffftau\frac{\tr (\bC_2-\bC_1)}{p}\frac{\tr \bC_a}{p}.
\end{align*}

Since $\bz_\bx$ is (standard) Gaussian and has the same distribution as $\bU\bz_\bx$ for any orthogonal matrix $\bU$ (i.e., such that $\bU^{\T}\bU=\bU\bU^{\T}=\iden$), we choose $\bU$ that diagonalize $\bA$ such that $\bA=\bU\bLambda\bU^{\T}$, with $\bLambda$ diagonal so that $\hat{g}(\bx)$ and $\tilde{g}(\bx)$ have the same distribution where
\begin{equation*}
\tilde{g}(\bx)=\bz_\bx^{\T}\bLambda\bz_\bx+\bz_\bx^{\T}\tilde{\bb}+c=\sum_{i=1}^{n}\left(z_i^2\lambda_i + z_i\tilde{b}_i+\frac{c}{n}\right)
\end{equation*}
and $\tilde{\bb}=\bU^{\T}\bb$, $\lambda_i$ the diagonal elements of $\bLambda$ and $z_i$ the elements of $\bz_\bx$.

Conditioning on $\bOmega$, we thus result in the sum of independent but not identically distributed random variables $r_i=z_i^2\lambda_i + z_i\tilde{b}_i+\frac{c}{n}$. We then resort to the Lyapunov CLT \cite[Theorem~27.3]{billingsley2008probability}.

We begin by estimating the expectation and the variance
\begin{align*}
\E[r_i|\bOmega]&=\lambda_i +\frac{c}{n}\\
{\rm Var}[r_i|\bOmega]&=\sigma_i^2=2\lambda_i^2+\tilde{b}_i^2
\end{align*}
of $r_i$, so that 
\begin{align*}
\sum_{i=1}^{n}\E[r_i|\bOmega]&=c_2-c_1+\gamma c_\bx\mathfrak{D}={\rm E}_a\\
s^2_n&=\sum_{i=1}^{n}{\sigma_i^2}=2\tr(\bA^2) + \bb^{\T}\bb\\
&=8\gamma^2c_1^2c_2^2\left(\ffftau\right)^2\frac{\left(\tr\left(\bC_2-\bC_1\right)\right)^2}{p^2}\frac{\tr \bC_a^2}{p^2}\\
&+4\gamma^2\left(\frac{\fftau}{n}\right)^2\by^{\T}\bP\bOmega^{\T}\frac{\bC_a}{p}\bOmega\bP\by\\
&+\frac{16\gamma^2c_1^2c_2^2(\fftau)^2}{p}(\bmu_2-\bmu_1)^{\T}\frac{\bC_a}{p}(\bmu_2-\bmu_1)\\
&+O(n^{-\frac{5}{2}}).
\end{align*}

We shall rewrite $\bOmega$ into two blocks as:
\begin{equation*}
\bOmega=
\begin{bmatrix}
\frac{\left(\bC_1\right)^{\frac{1}{2}}}{\sqrt{p}}\bZ_1,& \frac{\left(\bC_2\right)^{\frac{1}{2}}}{\sqrt{p}}\bZ_2
\end{bmatrix}
\end{equation*}
where $\bZ_1\in\mathbb{R}^{p\times n_1}$ and $\bZ_2\in\mathbb{R}^{p\times n_2}$ with i.i.d. Gaussian entries with zero mean and unit variance. Then 
\begin{equation*}
\bOmega^{\T}\frac{\bC_a}{p}\bOmega=\frac{1}{p^2}
\begin{bmatrix}
\bZ_1^{\T}(\bC_1)^{\frac{1}{2}}\bC_a(\bC_1)^{\frac{1}{2}}\bZ_1 & \bZ_1^{\T}(\bC_1)^{\frac{1}{2}}\bC_a(\bC_2)^{\frac{1}{2}}\bZ_2 \\  \bZ_2^{\T}(\bC_2)^{\frac{1}{2}}\bC_a(\bC_1)^{\frac{1}{2}}\bZ_1 & \bZ_2^{\T}(\bC_2)^{\frac{1}{2}}\bC_a(\bC_2)^{\frac{1}{2}}\bZ_2
\end{bmatrix}
\end{equation*}
and with $\bP\by=\by-(c_2-c_1)\mathbf{1}_n$, we deduce
\begin{align*}
&\by^{\T}\bP\bOmega^{\T}\frac{\bC_a}{p}\bOmega\bP\by=\frac{4}{p^2}\left(c_2^2\mathbf{1}_{n_1}^{\T}\bZ_1^{\T}(\bC_1)^{\frac{1}{2}}\bC_a(\bC_1)^{\frac{1}{2}}bZ_1\mathbf{1}_{n_1}\right.\\
&\left.-2c_1c_2\mathbf{1}_{n_1}^{\T}\bZ_1^{\T}(\bC_1)^{\frac{1}{2}}\bC_a(\bC_2)^{\frac{1}{2}}\bZ_2\mathbf{1}_{n_2} \right. \\
& \left. +c_2^2\mathbf{1}_{n_1}^{\T}\bZ_2^{\T}(\bC_2)^{\frac{1}{2}}\bC_a(\bC_2)^{\frac{1}{2}}\bZ_2\mathbf{1}_{n_2} \right).
\end{align*}

Since $\bZ_i\mathbf{1}_{n_i}\sim\mathcal{N}(\mathbf{0},n_i\mathbf{I}_{n_i})$, by applying the trace lemma \cite[Lemma~B.26]{bai2010spectral} we get
\begin{equation}
\by^{\T}\bP\bOmega^{\T}\frac{\bC_a}{p}\bOmega\bP\by-\frac{4nc_1^2c_2^2}{p^2}\left(\frac{\tr \bC_1\bC_a}{c_1}+\frac{\tr \bC_2\bC_a}{c_2}\right)\asc 0.\label{eq:Omega-limit}
\end{equation}

Consider now the events 
\begin{align*}
E&=\left\{ \left|\by^{\T}\bP\bOmega^{\T}\frac{\bC_a}{p}\bOmega\bP\by-\rho\right|<\epsilon\right\}\\
\bar{E}&=\left\{\left|\by^{\T}\bP\bOmega^{\T}\frac{\bC_a}{p}\bOmega\bP\by-\rho\right|>\epsilon\right\}
\end{align*}
for any fixed $\epsilon$ with $\rho=\frac{4nc_1^2c_2^2}{p^2}\left(\frac{\tr \bC_1\bC_a}{c_1}+\frac{\tr \bC_2\bC_a}{c_2}\right)$ and write
\begin{align}
&\E\left[\exp\left(iun\frac{\tilde{g}(\bx)-{\rm E}_a}{s_n}\right)\right]=\E\left[\exp\left(iun\frac{\tilde{g}(\bx)-{\rm E}_a}{s_n}\right)\bigg|E\right]\nonumber\\
&{\rm P}(E)+\E\left[\exp\left(iun\frac{\tilde{g}(\bx)-{\rm E}_a}{s_n}\right)\bigg|\bar{E}\right]{\rm P}(\bar{E})\label{eq:condition}
\end{align}

We start with the variable $\tilde{g}(\bx)|E$ and check that Lyapunov's condition for $\bar{r}_i=r_i-\E[r_i]$, conditioning on $E$,
\begin{equation*}
\lim_{n\to\infty}{\frac {1}{s_n^4}\sum _{i=1}^n\E[|\bar{r}_i|^4]}=0
\end{equation*}
holds by rewriting
\begin{equation*}
\lim_{n\to\infty}{\frac {1}{s_n^4}\sum _{i=1}^n\E[|\bar{r}_i|^4]}=\lim_{n\to\infty}{\sum _{i=1}^n\frac{60\lambda_i^4+12\lambda_i^2\tilde{b_i}^2+3\tilde{b_i}^4}{s_n^4}}=0
\end{equation*}
since both $\lambda_i$ and $\tilde{b}_i$ are of order $O(n^{-3/2})$.

As a consequence of the above, we have the CLT for the random variable $\tilde{g}(\bx)|E$, thus 
\[
\E\left[\exp\left(iun\frac{\tilde{g}(\bx)-{\rm E}_a}{s_n}\right)\bigg|E\right]\to\exp(-\frac{u^2}{2}).
\]

Next, we see that the second term in \eqref{eq:condition} goes to zero because $\big|\E\big[\exp\left(iun\frac{\tilde{g}(\bx)-{\rm E}_a}{s_n}\right)\big|\bar{E}\big]\big|\le 1$ and ${\rm P}(\bar{E})\to 0$ from \eqref{eq:Omega-limit} and we eventually deduce
\[\E\left[\exp\left(iun\frac{\tilde{g}(\bx)-{\rm E}_a}{s_n}\right)\right]\to\exp(-\frac{u^2}{2}).
\]

With the help of L\'evy's continuity theorem, we thus prove the CLT of the variable $n\frac{\tilde{g}(\bx)-{\rm E}_a}{s_n}$. Since $s_n^2\to{\rm Var}_a$, with Slutsky's theorem, we have the CLT for $n\frac{\tilde{g}(\bx)-{\rm E}_a}{\sqrt{{\rm Var}_a}}$ (thus for $n\frac{\hat{g}(\bx)-{\rm E}_a}{\sqrt{{\rm Var}_a}}$), and eventually for $n\frac{{g}(\bx)-{\rm E}_a}{\sqrt{{\rm Var}_a}}$ by Theorem~\ref{thm:Random Equivalent} which completes the proof.

\end{document}